\documentclass[twoside,11pt]{article}

\usepackage{blindtext}

\usepackage[preprint]{jmlr2e} %
\usepackage{tikz}
\usepackage{enumitem}
\usepackage{thmtools}
\usepackage{microtype}
\usepackage{booktabs} 
\usepackage{bbm}
\usepackage{pdflscape}
\usetikzlibrary{decorations.pathreplacing} %
\usepackage[dvipsnames]{xcolor}
\usepackage{mathrsfs}

\hypersetup{
  pdfborder={0 0 0}
}

\zcsetup{cap}
\zcRefTypeSetup{Setting}{
  name-sg=setting, name-pl=settings,
  Name-sg=Setting, Name-pl=Settings,
}

\zcRefTypeSetup{Assumption}{
  name-sg=assumption, name-pl=assumptions,
  Name-sg=Assumption, Name-pl=Assumptions,
}

\newcounter{para}
\zcRefTypeSetup{para}{Name-sg={},Name-pl={}}

\zcRefTypeSetup{item}{
  name-sg={}, name-pl={},
  Name-sg={}, Name-pl={}
}
\zcRefTypeSetup{equation}{
  Name-sg = {Equation},
  name-sg = {equation},
  Name-pl = {Equations},
  name-pl = {equations},
  refbounds = { , , , } %
}

\newcommand{\indep}{\perp \!\!\! \perp}



\renewcommand{\paragraph}[1]{\par\noindent\textbf{#1.}\hspace{0.4em}}

\usepackage{lastpage}
\jmlrheading{}{2026}{1-\pageref{LastPage}}{04/26; Revised }{}{21-0000}{Vollmer, Weichwald and Pfister}

\ShortHeadings{SPICE}{Vollmer, Pfister$^*$ and Weichwald$^*$}
\firstpageno{1}

\usepackage{scrextend}
\makeatletter
\def\maketitle{\par
 \begingroup
   \def\thefootnote{\fnsymbol{footnote}}
   \def\@makefnmark{\hbox to 0pt{$^{\@thefnmark}$\hss}}
   \deffootnote[1.7em]{1.6em}{2em}{$^\thefootnotemark$}
   \@maketitle \@thanks
 \endgroup
\setcounter{footnote}{0}
 \let\maketitle\relax \let\@maketitle\relax
 \gdef\@thanks{}\gdef\@author{}\gdef\@title{}\let\thanks\relax}
\makeatother

\definecolor{MColor}{RGB}{120,100,232}
\definecolor{PColor}{RGB}{255,0,170}
\definecolor{IColor}{RGB}{255,85,0}

\begin{document}

\title{Identifying Causal Effects Using a Single Proxy Variable}

\author{\name Silvan Vollmer \email sivo@math.ku.dk \\
       \addr Department of Mathematical Sciences\\
       University of Copenhagen\\
       Copenhagen, Denmark
       \AND
       \name Niklas Pfister\footnotemark[1]\phantom{\thanks{Authors contributed equally.}} \email niklas.pfister@gmail.com\\
       \addr Lakera AI\\
       Zurich, Switzerland
       \AND
       \name Sebastian Weichwald\footnotemark[1] \email sweichwald@math.ku.dk \\
       \addr Department of Mathematical Sciences\\
       University of Copenhagen\\
       Copenhagen, Denmark
       }

\maketitle

\begin{abstract}%
Unobserved confounding is a key challenge when estimating causal effects from a treatment on an outcome. In this work, we assume that we observe a single, potentially multi-dimensional proxy variable of the unobserved confounder and that we know the mechanism that generates the proxy from the confounder. Under an assumption called Single Proxy Identifiability of Causal Effects or simply SPICE, we prove that this error mechanism is complete and causal effects are identifiable. We extend the proxy-based causal identifiability results by \citet{Pearl2010,Kuroki2014} to multi-dimensional continuous settings, more flexible functional relationships and a broader class of distributions. Further, we develop a neural network based estimation framework, SPICE-Net, to estimate causal effects, which is applicable to both discrete and continuous treatments.
\end{abstract}


\section{Introduction}

Unobserved confounding is a fundamental problem for estimating causal effects from observational data \citep{Byrnes2025,Lu2009,VanderWeele2017}. A confounder is a common cause of both the treatment and the outcome, inducing non-causal dependencies between them. In practice, it is unrealistic to assume complete knowledge of all confounders. Instead, we may rely on a proxy variable of the true confounder, such as a noisy measurement. We assume that we observe a single, potentially multi-dimensional proxy variable of the confounder and assume that the error mechanism that generates the proxy variable from the confounder is known. In this setting, our contribution is twofold. First, we prove a novel causal identifiability condition, which we call Single Proxy Identifiability of Causal Effects or simply SPICE, that extends existing proxy-based causal identifiability results. Second, we propose a neural network architecture and corresponding loss function, which allows us to estimate causal effects.

Existing identifiability results can be divided into a single proxy and a multiple proxy setting. In the single proxy setting, we assume that we observe a potentially multi-dimensional proxy variable whose components need not be conditionally independent given the confounder. In the multiple proxy case, we observe multiple proxies, which are assumed to be conditionally independent given the confounder. Our work falls into the single proxy setting. In this setting, \citet{Pearl2010,Kuroki2014} show that causal effects are identifiable for discrete confounders and proxies and for linear Gaussian structural causal models by using the matrix adjustment method of \citet{Greenland2008}.
Crucially, \citet{Pearl2010,Kuroki2014} assume knowledge of the error mechanism that links the proxy to the confounder, that is, the set of conditional densities of the proxy given the confounder. We extend these results both in the discrete case by allowing the proxy to have a higher dimension than the confounder and in the continuous case by allowing for multi-dimensional and non-Gaussian variables. Further, we only require that the proxy is an injective function of the confounder plus noise, rather than a linear model across all variables as assumed by  \citet{Pearl2010,Kuroki2014}. Without assuming knowledge on the error mechanism, \citet{Park2024} show identifiability of the average causal effect for the treated of a binary treatment using a single proxy, thereby generalising the control outcome calibration approach \citep{TchetgenTchetgen2014} and ultimately the Difference-in-Difference approach \citep{Card1993}. \citet{Xu2025} extend the framework of \citet{Park2024} to continuous treatments while assuming that the outcome is a deterministic function of the treatment and the confounder. 

With multiple proxies, causal effect identifiability can be achieved through the proximal causal inference framework \citep{Miao2018}, with recent overviews provided by \citet{Ringlein2025,TchetgenTchetgen2024}. In proximal causal inference, unobserved confounding is addressed by classifying proxies into treatment and outcome proxies. Identifiability in proximal causal inference is achieved via bridge functions. As an alternative to proximal causal inference, array decomposition approaches \citep{Pearl2010,Kuroki2014,Deaner2023} recover the joint distribution over unobserved and observed variables up to injective transformations of the unobserved variables \citep{Guo2025}. It originates from the work by \citet{Kruskal1977} and for continuous variables, \citet{Deaner2023} achieves identifiability of causal effects using three proxies by building on the results of \citet{Hu2008}. Our work can be seen as an array decomposition approach that assumes a known error mechanism, as in \citet{Pearl2010,Kuroki2014}, and a single proxy variable in contrast to \citet{Deaner2023}.

Based on our identifiability result, we propose SPICE-Net, a machine learning method to estimate causal effects. SPICE-Net handles both discrete and continuous treatments and outcomes and it builds on Engression \citep{Shen2023}. We show that it recovers the unobserved confounder distribution up to an injective transformation and can be combined with modern nonparametric estimators of causal effects. We compare our method to a variational autoencoder by \citet{Louizos2017} and kernel-based methods by \citet{Xu2025}. SPICE-Net is available in the Python package \href{https://pypi.org/project/spiceaisle/}{\texttt{spiceaisle}}.

The remainder of the paper is structured as follows. In \zcref{se: PCSCM}, we define the proximal-confounded structural causal model and use it to define our target of inference. We then introduce SPICE \zcref{assumption: SPICE} and prove our identifiability results (\zcref{thm: discrete confounder} and \zcref{thm: continuous confounder}) in \zcref{se: identifiability}. In \zcref{se: method}, we propose our neural network based estimator SPICE-Net and apply it to simulated and real-world data in \zcref{se: application}. Finally, we discuss extensions and applications of SPICE and conclude in \zcref{se: discussion}. 
 
\paragraph{Notation} 
For a measurable space $\mathcal{X}$ and a measure $\mu$ on $\mathcal{X}$, we denote the space of $\mu$-equivalent measurable functions from $\mathcal{X}$ to $\mathbb{R}$ that are absolutely integrable by $L_1 (\mathcal{X})$. For a measurable product space $\mathcal{X}\times\mathcal{Y}$ with measure $\mu$ and random variables $(X, Y)\in\mathcal{X}\times\mathcal{Y}$, we denote the joint distribution by $P_{X, Y}$, the density with respect to $\mu$ by $p_{X,Y}$ --~assuming the existence of densities throughout the paper~-- and the expectation by $\mathbb{E}$. Furthermore, for all $y\in\mathcal{Y}$, the conditional distribution of $X$ given $Y=y$ is denoted by $P_{X|Y=y}$ and the conditional density with respect to $\mu$ by $p_{X|Y}(\cdot\mid y)$. We consider distributions under interventions using $do$-notation \citep{Pearl2009}. Finally, a normal distribution with mean zero and variance one is denoted by $\mathcal{N}(0,1)$ and an exponential distribution with scale one by $\text{Exp}(1)$.

\section{Proximal-Confounded Structural Causal Model}
\label{se: PCSCM}
We now present our causal model describing the treatment effect on an outcome under unobserved confounding. Specifically, we consider a Proximal-Confounded Structural Causal Model (PC-SCM), which assumes that the treatment, the outcome and a noisy proxy of the confounder are observed. A formal description of the model is provided in the following definition.

\begin{definition}[Proximal-Confounded Structural Causal Model (PC-SCM)]
\label{def: pcscm}
    Consider a treatment $X \in \mathcal{X} \subseteq\mathbb{R}^p$, an outcome $Y \in \mathcal{Y} \subseteq\mathbb{R}$, an unobserved confounder $U \in \mathcal{U} \subseteq \mathbb{R}^k$ and a proxy variable $W \in \mathcal{W} \subseteq\mathbb{R}^d$, where $p, k, d \in \mathbb{N}$.\footnote{The setting can be extended to a multi-dimensional outcome $Y \in \mathcal{Y} \subseteq\mathbb{R}^q$ with $q \in \mathbb{N}$.} A proximal-confounded structural causal model (PC-SCM) $M$ is a structural causal model that induces the directed acyclic graph (DAG) $G$, both shown in \zcref{fig: scm and dag}, with measurable functions $f_W, f_X, f_Y$ and mutually independent random noise terms $(N_U, E, N_X, N_Y)$.\footnote{The mutual independence assumption can be relaxed as long as the independence statement in \zcref{eq: independence statement} holds.} 
    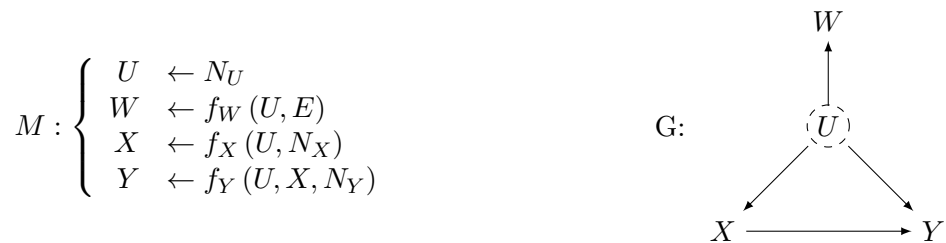
\begin{figure}[!htbp]
        \centering
        \begin{minipage}{0.48\linewidth}
            \[
            M:\left\{
            \begin{array}{rl}
            U & \leftarrow N_U \\
            W & \leftarrow f_W\left(U, E\right) \\
            X & \leftarrow f_X \left(U, N_X\right) \\
            Y & \leftarrow f_Y \left(U, X, N_Y\right)
            \end{array}
            \right.
            \]
        \end{minipage}
        \hfill
        \begin{minipage}{0.48\linewidth}
            \centering
            \begin{tikzpicture}[scale=0.7, >=stealth, node distance=2cm]
                \node at (-3,0) {G:};

                \node (U) at (0,0) {$U$};
                \node (W) at (0,2) {$W$};
                \node (X) at (-2,-2) {$X$};
                \node (Y) at (2,-2) {$Y$};

                \draw[-latex] (U) -- (W);
                \draw[-latex] (U) -- (X);
                \draw[-latex] (U) -- (Y);
                \draw[-latex] (X) -- (Y);

                \draw[dashed] (U) circle (0.4cm);
            \end{tikzpicture}
        \end{minipage}
        \caption{Structural causal model $M$ and directed acyclic graph $G$ of \zcref{def: pcscm}.}
        \label{fig: scm and dag}
    \end{figure}
    
    Moreover, it induces a unique, well-defined probability distribution on the variables $(U,W,X,Y)$ that we denote by $P^M$ \citep[for example][]{Bongers2021} and for each intervention $do(X=x)$ that sets $X$ to a fixed value $x\in\mathcal{X}$, an interventional distribution $P^{M; do (X := x)}_Y$ on $Y$. Furthermore, the implied joint distribution $P^M$ is absolutely continuous with respect to a $\sigma$-finite product measure $\mu$ such that it has a density and the density is defined on $\mathbb{R}^k\times\mathbb{R}^d\times\mathbb{R}^p\times\mathbb{R}$
    and has support on $\mathcal{U} \times \mathcal{W} \times \mathcal{X} \times \mathcal{Y}$, that is
    \begin{align}
        p^M(u,w,x,y) > 0
        \iff
        (u,w,x,y)\in
    \mathcal{U} \times \mathcal{W} \times \mathcal{X} \times \mathcal{Y}.
    \end{align}
\end{definition}

\begin{Setting}
\label{setting: pcscm}
    We assume that we observe $n \in \mathbb{N}$ independent and identically distributed copies $(W_1,X_1,Y_1),\dots,(W_n,X_n,Y_n)$ of the random variable $(W,X,Y)$ generated by a PC-SCM $M_0$. We assume that the error mechanism $\{ p^{M_0}_{W \mid U} (\cdot \mid u)\mid u \in \mathcal{U} \}$ is known. Moreover, we assume that the causal function in \zcref{def: causal function and ACE} is uniquely defined and integrable. We denote the density $p^{M_0}_{W,X,Y}$ as the observational density.
\end{Setting}

Using d-separations, it follows for all PC-SCMs that the proxy is conditionally independent of the treatment and the outcome given the confounder, that is,
    \begin{align}
    \label{eq: independence statement}
        W \indep (X,Y) \mid U.
    \end{align}
We conclude that for almost all $(u, x, y) \in \mathcal{U} \times \mathcal{X} \times \mathcal{Y}$ it holds that $p^{M}_{W \mid U, X, Y }(\cdot \mid u,x,y) = p^{M}_{W \mid U}(\cdot \mid u)$.

Given \zcref{setting: pcscm}, we aim to determine how intervening on the treatment $X$ impacts the outcome $Y$. Formally, this is captured by the causal function, which can be contrasted to provide a measure of the causal effect of $X$ on $Y$.

\begin{definition}[causal function and average causal effect]
    \label{def: causal function and ACE}
    Let $M_0$ be the true PC-SCM from \zcref{setting: pcscm}. We define the causal function $\theta^{M_0}:\mathcal{X}\rightarrow\mathbb{R}$ for all $x \in \mathcal{X}$ as
    \begin{equation}
        \label{eq: causal function}
       \theta^{M_0}(x) := \mathbb{E}^{M_0; do(X:=x)}\left[Y \right] = \int y \int  p^{M_0}_{Y \mid U,X} \left(y \mid u, x \right) p^{M_0}_{U}(u) \, \mu(du, dy),
    \end{equation}
    where the equality holds by the adjustment formula \citep[Theorem~3.3.2]{Pearl2009}.
    Moreover, the average causal effect (ACE) for a univariate, continuous treatment $X \in \mathbb{R}$ is defined by
    \begin{align}
        \operatorname{ACE}^{M_0} &:= \mathbb{E}^{M_0}\left[ \tfrac{\partial}{\partial x} \theta^{M_0}(x) \mid_{x=X}\right]
         \label{eq: ACE continuous}
    \end{align}
    and for a univariate, binary treatment $X \in \{0,1\}$, it is defined by
    \begin{align}
        \operatorname{ACE}^{M_0} &:= \theta^{M_0}(1) - \theta^{M_0}(0).      \label{eq: ACE binary}   
    \end{align}
\end{definition}

As shown in the definition, the causal function can be identified by controlling for the confounding $U$ using for example the adjustment formula. Using $W$ instead of $U$ is however in general not possible and will result in a biased estimate. In \zcref{app:linear_Gaussian_setting}, we derive an explicit expression for the difference between controlling for $W$ and controlling for $U$ in a linear Gaussian setting. From this expression we can see that the difference is small if the proxy is close to $U$ or if the direct effect of $X$ on $Y$ dominates the confounding due to $U$.

\section{Identifiability of the Causal Function}
\label{se: identifiability}

We now study the identifiability of the causal function. The causal function is identifiable if it is the same across all PC-SCMs that induce the same observational features. 
For a PC-SCM $M$ defined in \zcref{def: pcscm}, the observational features consist of both the observational density and the error mechanism and  we denote them by
\begin{align}
    \mathcal{F}^{M} := \{ p^{M}_{W,X,Y}\} \cup \{p_{W \mid U}^{M} (\cdot \mid u) \mid u \in \mathcal{U}\}.
\end{align}
PC-SCMs that share the same observational features are referred to as observationally equivalent. This leads us to the following definition of the identifiability of the causal function \citep{Pearl2014,Pearl2009}.

\begin{definition}[identifiability]
\label{def: identifiability}
    Let $M_0$ be the true PC-SCM from \zcref{setting: pcscm}. We denote the set of PC-SCMs that are observationally equivalent to $M_0$ by $\mathcal{M} (M_0)$, that is, for all $M \in \mathcal{M} (M_0)$, it holds that
    \begin{equation}
    \label{eq: obs features same}
        \mathcal{F}^M = \mathcal{F}^{M_0}\quad\text{a.e.}
    \end{equation}
    The causal function $\theta^{M_0}$ is identifiable if for all models $M \in \mathcal{M} (M_0)$, it holds that
    \begin{equation}
    \label{eq: identifiability unique theta}
        \theta^{M} \equiv \theta^{M_0}\quad\text{a.e.}
    \end{equation}
\end{definition}

Our main identifiability result relies on a technical assumption on the error mechanism defined below, which ensures that the proxy $W$ contains sufficient information about the confounder $U$.

\begin{definition}[complete error mechanism]
\label{def: complete error}
    Let $M_0$ be the true PC-SCM from \zcref{setting: pcscm}. We say that the error mechanism $\{ p_{W \mid U }^{M_0}(\cdot \mid u) \mid u \in \mathcal{U}\}$ is $L_1$-complete if for all $\delta\in L_1(\mathcal{U})$
            \begin{equation}
            \label{eq: definition complete error l1}
            \int_{\mathcal{U}} p^{M_0}_{W \mid U} (\cdot \mid u)\delta (u) \, \mu(du) \equiv 0\quad\text{a.e.}
            \quad\Longrightarrow\quad
            \delta\equiv 0\quad\text{a.e.}
            \end{equation}
\end{definition}
\zcref{def: complete error} is similar to the completeness assumption of \citet{Mattner1993} as outlined by \citet[HS Assumption 3]{Deaner2023}. Assuming that the error mechanism is complete, we get the following identifiability result.

\begin{restatable}[identifiability of the causal function]{theorem}{thmidentifiability}
\label{thm: effect identifiable}
    Let $M_0$ be the true PC-SCM from \zcref{setting: pcscm}. Assume that the error mechanism is $L_1$-complete as in \zcref{def: complete error}.
    \newline
    Then, the causal function $\theta^{M_0}$ is identifiable.
\end{restatable}

We prove \zcref{thm: effect identifiable} in \zcref{app:proof_theorem}. Next, we investigate sufficient conditions for the error mechanism to be complete and thus for \zcref{thm: effect identifiable} to apply. We distinguish between two cases in the following subsections, based on whether the confounders and proxies are discrete or continuous random variables. In both cases, we require that the dimension of the proxy is at least as large as the dimension of the confounder. In the discrete case, the error mechanism is complete if the matrix with entries $p^{M_0}_{W\mid U}(\cdot\mid\cdot)$ has full column rank. In the continuous case, completeness holds for an injective proxy-confounder relation with additive noise whose Fourier transform has no zeros.

\subsection{Discrete Confounders and Proxies}
\label{subse: discrete confounders}

We begin with the discrete case and assume that the proxy has at least the dimension of the unobserved confounder, that is, $d\geq k$. Then, the error mechanism is complete if there exists a matrix of conditional probabilities of the proxy given the confounder with full column rank. The formal result is provided in the following theorem.

\begin{restatable}[a complete error mechanism for discrete confounders]{theorem}{thmdiscrete}
\label{thm: discrete confounder}
    Let $M_0$ be the true PC-SCM from \zcref{setting: pcscm}. Further, assume that the following holds.
    \begin{enumerate}[label=(\roman*)]
    \item \label{as: discrete confounder} (discrete support and sufficient dimension) We have a discrete confounder $U \in \mathcal{U} =\{u_1, \dots, u_{k}\}$ and a proxy $W \in \mathcal{W} = \{w_1, \dots, w_{d}\}$ with $d \geq k$.
    \item \label{as: invertibilty} (full rank error mechanism) We assume that there exists a subset $\mathcal{W}_r:=\{w_1', \dots, w_r'\} \subseteq \mathcal{W}$ with cardinality $r \in \mathbb{N}$ and $k \leq r \leq d$ such that the matrix
    \begin{equation}
    \label{eq: error matrix}
        \left(\begin{array}{ccc}
        p^{M_0}_{W \mid U} (w_1' \mid u_1) & \cdots & p^{M_0}_{W \mid U} (w_1' \mid u_{k}) \\
        \vdots & \ddots & \vdots \\
        p^{M_0}_{W \mid U} (w_r' \mid u_1) & \cdots & p^{M_0}_{W \mid U} (w_r' \mid u_{k})
        \end{array}\right)
        \in [0,1]^{r \times k}
    \end{equation}
    has full column rank.
\end{enumerate}
    Then, the error mechanism $\{ p_{W \mid U }^{M_0}(\cdot \mid u) \mid u \in  \{u_1,\dots,u_k\}\}$ is $L_1$-complete.
\end{restatable}

A proof is given in \zcref{app:proof_corollary_complete_discrete}. Together with \zcref{thm: effect identifiable}, the result implies identifiability of the causal function. \zcref{thm: discrete confounder} extends results of \citet{Pearl2010,Kuroki2014}, which we state and prove in \zcref{app:proofs_old_results} using our notation and formal framework to facilitate comparison, by relaxing their assumption that $d=k$. 

\subsection{Continuous Confounders and Proxies}
\label{subse: continuous confounders}

Now we turn to continuous confounders and proxies.
We regard \zcref{assumption: SPICE} in combination with \zcref{thm: continuous confounder} as the most practically relevant theoretical result of this paper: It provides sufficient conditions for the error mechanism to be complete. We refer to \zcref{assumption: SPICE} as Single Proxy Identifiability of Causal Effects or simply SPICE.

\begin{restatable}[SPICE]{Assumption}{spiceassumption}
\label{assumption: SPICE}
    Let $M_0$ be the true PC-SCM from \zcref{setting: pcscm}.
    \begin{enumerate}[label=(\roman*)]
        \item \label{as: continuous confounder and proxy} (sufficient proxy dimension) 
        The proxy $W \in \mathcal{W} \subseteq \mathbb{R}^d$ has at least the dimension of the confounder $U \in \mathcal{U} \subseteq \mathbb{R}^k$, that is, $d \geq k$.
        \item \label{as: additive error} (additive noise model)
        The proxy follows an additive noise model where $f_W: (U,E) \mapsto b(U) + E$ for a 
        Borel measurable and injective function $b: \mathbb R^k \rightarrow \mathbb R^d$
        and a random variable $E \in\mathcal{E}\subseteq \mathbb{R}^d$, which is independent of $U$, that is,
        \begin{align}
            W = b(U)+ E \quad\text{with}\quad U\indep E.
        \end{align}
        \item \label{as: positive density} (density) The distribution of $E$ has a density $p_E^{M_0}$ on $\mathbb{R}^d$ with respect to the Lebesgue measure.
        \item \label{as: fourier transform} (non-vanishing Fourier transform) The Fourier transform of the density of $E$ has no zeros, that is, we have for all $t\in\mathbb{R}^d$ that
        \begin{equation}
            \left( 2 \pi \right)^{- \frac{d}{2}} \int p^{M_0}_{E}(e) \exp( - i t^{T} e) \, de \neq 0
        \end{equation}
        with imaginary unit $i$.
    \end{enumerate}
\end{restatable}

If SPICE \zcref{assumption: SPICE} holds, then the error mechanism is complete by the following theorem.

\begin{restatable}[a complete error mechanism for continuous confounders]{theorem}{thmcontinuous}
\label{thm: continuous confounder}
    Let $M_0$ be the true PC-SCM from \zcref{setting: pcscm}. Further, assume that SPICE \zcref{assumption: SPICE} holds.
    \newline
    Then, the error mechanism $\{ p_{W \mid U }^{M_0}(\cdot \mid u) \mid u \in \mathcal{U}\}$ is $L_1$-complete.
\end{restatable}

We prove \zcref{thm: continuous confounder} using the injectivity of the Fourier transform in \zcref{app:proof_corollary_complete_continuous}. If SPICE \zcref{assumption: SPICE} holds, then the causal function is identifiable by \zcref{thm: continuous confounder} and \zcref{thm: effect identifiable}. The result in \zcref{thm: continuous confounder} extends the framework of \citet{Pearl2010,Kuroki2014}, which we state and prove in \zcref{app:proofs_old_results} using our notation and formalism, by considering higher dimensions, fewer structural assumptions and a broader class of noise distributions. \citet{Pearl2010,Kuroki2014} assume a linear Gaussian dependency between all variables, which are assumed to be one-dimensional. Further, they require mean-zero noise variables to achieve identifiability. Their case is subsumed by \zcref{thm: continuous confounder}. In their framework, $E \in \mathbb{R}$ follows a Gaussian distribution whose Fourier transform has no zeros \citep{Kossovsky2023} and hence the causal function is identifiable by \zcref{thm: effect identifiable}.

Next, we clarify which noise distributions have Fourier transforms with no zeros and satisfy SPICE \zcref{assumption: SPICE}-\zcref{as: fourier transform}. We define the characteristic function, infinitely divisible distributions and convolutionally infinitely divisible kernels in \zcref{app:characteristic_divisible}.

\begin{restatable}[sufficient conditions for a non-vanishing Fourier transform]{proposition}{prodivisible}
\label{prop: sufficient conditions}
    Let $M_0$ be the true PC-SCM from \zcref{setting: pcscm} and assume that SPICE \zcref{assumption: SPICE}-\zcref{as: positive density} holds.
    Further, assume that one of the following holds.
    \begin{enumerate}[label=(\roman*)]
        \item \label{as: non-zero characteristic function} The characteristic function of $E$ has no zeros. This means that for all $t \in \mathbb{R}^d$, we have that
    \begin{equation}
        \int p^{M_0}_{E} (e) \exp( i t^{T} e) \, de \neq 0
    \end{equation}
    with imaginary unit $i$.
    \item \label{as: infinite divisibility} The probability distribution $P^{M_0}_E$ is infinitely divisible. 
    \item \label{as: cid kernel} The density $p^{M_0}_E$ is a convolutionally infinitely divisible kernel.
    \end{enumerate}
    Then, SPICE \zcref{assumption: SPICE}-\zcref{as: fourier transform} holds, that is, the Fourier transform of the density of $E$ has no zeros. 
\end{restatable}

The second result in the proposition follows from \citet[Lemma 7.5]{Sato1999} as shown in \zcref{app:characteristic_divisible}. In particular, we have that  
\[
    \text{Prop.~\ref{prop: sufficient conditions}-} 
    \text{\zcref{as: cid kernel}} \quad\Longrightarrow\quad
    \text{Prop.~\ref{prop: sufficient conditions}-}\text{\zcref{as: infinite divisibility}} \quad\Longrightarrow\quad
    \text{Prop.~\ref{prop: sufficient conditions}-}\text{\zcref{as: non-zero characteristic function}}
\]
and
\[
    \text{Prop.~\ref{prop: sufficient conditions}-}\text{\zcref{as: non-zero characteristic function}}
    \quad\iff\quad \text{SPICE~\ref{assumption: SPICE}-\ref{as: fourier transform}}.
\]
Using this result, we can show that many common distributions satisfy SPICE \zcref{assumption: SPICE}\nobreakdash-\zcref{as: fourier transform}.
Focusing on \zcref{prop: sufficient conditions}-\zcref{as: infinite divisibility} and univariate distributions, we have that the normal, Cauchy, Laplace or gamma distribution, which includes
the exponential and chi-squared distribution, and all stable distributions are infinitely divisible as shown by \citet{Lukacs1970}. Furthermore, the log-normal distribution, all generalised hyperbolic distributions, such as the $t$-distribution and all generalised inverse Gaussian distributions are also infinitely divisible as shown by \citet{Thorin1977} and \citet{Barndorff1977}. Further, any convolution of two infinitely divisible probability distributions is again infinitely divisible by \citet[Lemma 7.4]{Sato1999}. For a multi-dimensional variable $E$, we have that the product measure of infinitely divisible probability measures, representing independent components of $E$, is again infinitely divisible as shown by \citet{Horn1978}. Gaussian and Cauchy distributions on $\mathbb{R}^d$ are infinitely divisible by \citet{Sato1999}. 

Now we consider \zcref{prop: sufficient conditions}-\zcref{as: cid kernel}, which reinforces the applicability of \zcref{thm: continuous confounder} in a multi-dimensional proxy-confounder setting. We provide examples for convolutionally infinitely divisible kernels based on \citet[Example 3.4]{Nishiyama2016}. Convolutionally infinitely divisible kernels include Gaussian, Laplace, Cauchy, for all $\alpha \in (0,2]$ $\alpha$-stable, sub-Gaussian $\alpha$-stable, Student's t, GH, normalized inverse Gaussian, variance gamma (for example Matérn) and tempered $\alpha$-stable kernels by \citet{Grosswald1976,Rachev2011,Rosinski2007,Bianchi2010}. Other convolutionally infinitely divisible kernels can be constructed due to their closure under convolution, as shown by \citet[Proposition 3.7]{Nishiyama2016}.

Next, we provide an example where the error mechanism is not complete. If the unobserved confounder enters the proxy equation in a non-injective way, the error mechanism fails to be complete. The loss of information induced by a non-injective function, such as when the confounder enters the proxy equation squared, results in a non-complete error mechanism. We define a non-injective function of positive and finite measure for this in \zcref{def: non-injective function} and prove the following result in \zcref{app_non injective function}.

\begin{restatable}[non-injective noise model is not complete]{proposition}{prononinjective}
\label{prop: non injective error mechanism}
    Let $M_0$ be the true PC-SCM from \zcref{setting: pcscm}. Assume the proxy follows an additive noise model where $f_W : (U,E) \mapsto g(U) + E$ for a measurable and non-injective function of positive and finite measure $g: \mathcal{U} \rightarrow \mathbb{R}^d$ and a 
    random variable $E \in \mathcal{E} \subseteq\mathbb{R}^d$, which is independent of $U$, that is, 
    \begin{align}
        W = g(U) + E \quad\text{with}\quad U\indep E.
    \end{align}
    Further, assume that the distribution of $E$ has a density on $\mathbb{R}^d$ with respect to the Lebesgue measure. \newline
    Then, the error mechanism is not $L_1$-complete.
\end{restatable}

\section{Estimation of the Causal Function}
\label{se: method}

To estimate the causal function, we develop SPICE-Net, a neural network based estimation framework. Throughout this section, we assume that the causal function is identifiable and, in particular, that SPICE \zcref{assumption: SPICE} holds. SPICE-Net builds on Engression, a distributional regression method developed by \citet{Shen2023}, and we show that it recovers the distribution of the unobserved confounder up to an injective transformation.

SPICE-Net is a two step procedure: First, we estimate the conditional density of the proxy $W$ given the treatment and outcome $(X, Y)$ using a neural network. Second, we sample observations from an injective transformation of the confounder and use these to adjust when estimating the causal function.

In the first step, we use the Engression framework by \citet{Shen2023} to estimate the density of $W$ given $(X,Y)$. We use a neural network with $l \in \mathbb{N}$ linear layers as visualised in \zcref{fig: spicenet}. We input the treatment and the outcome into the neural network and append noise nodes $(\varepsilon_1,\dots,\varepsilon_{l-2})$ to the first $l-2$ layers. Each noise node has a user-specified dimension and we collect them as $\varepsilon := (\varepsilon_1,\dots,\varepsilon_{l-2})$, where $\varepsilon$ follows a multivariate standard Gaussian distribution, that is a multivariate Gaussian distribution with zero mean and identity covariance. We use rectified linear unit (ReLU) activation functions up to layer $l-2$ followed by a linear activation. In the second-to-last layer, we add independent samples $e$ of dimension $d$ from the distribution of $E$. We know the distribution of $E$ since we know the error mechanism, that is, we know by SPICE \zcref{assumption: SPICE}-\zcref{as: continuous confounder and proxy,as: additive error,as: positive density} for a Borel measurable and injective function $b: \mathbb R^k \rightarrow \mathbb{R}^{d}$ and for almost all $(u,w)\in\mathcal{U}\times\mathcal{W}$ that
\begin{align}
        p^{M_0}_{W \mid U}(w \mid u ) = p^{M_0}_{E\mid U}(w-b(u)\mid u) = p^{M_0}_E(w-b(u)),
\end{align}
which holds by the change of variables formula (see for example Equation 2.89 in \citet{Murphy2012}) and the independence of $E$ and $U$. The neural network in \zcref{fig: spicenet} can be seen as a function $g_\gamma: (x,y,\varepsilon,e) \mapsto w$, with the output denoted by $w$, which approximates the conditional density of $W \mid (X,Y)$. We define a corresponding model class as the set of all functions by
\begin{align}
    \mathcal{G} = \{g_\gamma: (x,y,\varepsilon,e) \mapsto w \mid \gamma \text{ valid neural network parameters}\}.
\end{align}
\begin{figure}[!htbp]
    \centering
    \begin{tikzpicture}[shorten >=1pt, ->, draw=black!50, node distance=\layersep]
    
        \def\layersep{2.5cm}

        \tikzstyle{neuron} = [circle, draw=black, minimum size=18pt, inner sep=0pt, text centered, align=center, minimum width=22pt, minimum height=22pt]
        \tikzstyle{dots} = [inner sep=0pt, font=\huge]
        
        \node[neuron] (I-1) at (0,-1) {$x$};
        \node[neuron] (I-2) at (0,-2) {$y$};
        \node[neuron, draw=gray] (I-3) at (0,-3) {$\varepsilon_1$};

        \node[neuron] (H1-1) at (0.7*\layersep,-1.5) {};
        \node[neuron, draw=gray] (H1-2) at (0.7*\layersep,-2.5) {$\varepsilon_2$};

        \coordinate (V1-1) at (1.2*\layersep,-1.5) {};
        \coordinate (V1-2) at (1.2*\layersep,-2.5) {};

        \node[dots] (dots) at (1.3*\layersep, -2) {$\cdots$};

        \coordinate (V2-1) at (1.4*\layersep,-1.5) {};
        \coordinate (V2-2) at (1.4*\layersep,-2.5) {};

        \node[neuron] (H2-1) at (1.9*\layersep,-1.5) {};
        \node[neuron, draw=gray] (H2-2) at (1.9*\layersep,-2.5) {$\varepsilon_{l-2}$};

        \node[neuron] (O-1) at (2.6*\layersep, -2) {};

        \node[neuron] (O-2) at (3.6*\layersep, -2) {$w$};

        \foreach \i in {1,...,3}
            \foreach \j in {1}
                \draw[-latex] (I-\i) -- (H1-\j);

        \foreach \i in {1,...,2}
            \foreach \k in {1}
                \draw[-] (H1-\i) -- (V1-\k);

        \foreach \i in {1,...,2}
            \foreach \j in {1}
                \draw[-latex] (V2-\i) -- (H2-\j);

        \foreach \i in {1,...,2}
            \draw[-latex] (H2-\i) -- (O-1);

        \draw[-latex] (O-1) -- node[above] {$+e$} (O-2);

        \draw[-, decorate, decoration={brace,mirror,amplitude=10pt}] 
        ([xshift=-0.3cm, yshift=-0.5cm] I-3.south) 
        -- ([xshift=0.3cm, yshift=-0.5cm] O-2.south |- I-3.south)
        node[midway,below=10pt] {$g_{\gamma}$};
        
    \end{tikzpicture}
    \caption{The first step of SPICE-Net is a neural network that builds on the Engression framework by \citet{Shen2023}. It has weights $\gamma$ and it takes the treatment and the outcome as inputs and appends independent standard Gaussian noise nodes $\varepsilon_1,\dots,\varepsilon_{l-2}$ to the first $l-2$ layers. We add samples from the distribution of $E$ in the second-to-last layer and denote the output by $w$. It recovers the conditional distributions of interest by \zcref{prop: population guarantee SPICE-Net}.} 
            \label{fig: spicenet}
\end{figure}
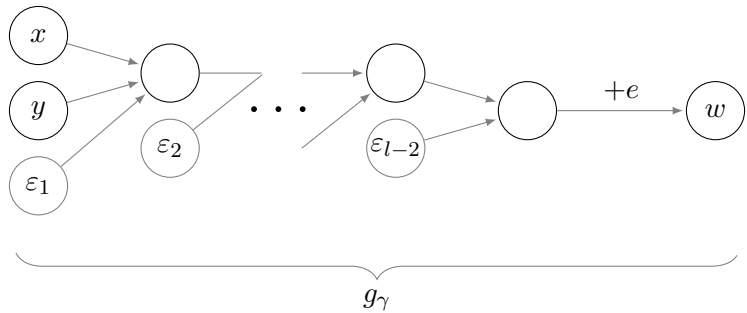

We optimise the weights of our neural network minimising the energy loss as in \citet{Shen2023}, which is based on the energy score by \citet{Gneiting2007}. This leads to the following population-level solution
\begin{align}
\label{eq: population version SPICE-Net}
    g^*_{\gamma} \in \operatorname*{argmin}_{g_\gamma \in \mathcal{G}}\ \mathbb{E}\left[ || W - g_{\gamma}(X,Y,\varepsilon,e)|| - \frac{1}{2} || g_{\gamma}(X,Y,\varepsilon,e) - g_{\gamma}(X,Y,\varepsilon',e')|| \right],
\end{align}
where $\varepsilon, \varepsilon'$ are vectors drawn independently from a multivariate standard Gaussian distribution and $e, e'$ are independent draws from the distribution of $E$. If the model is correctly specified, we show that $g_\gamma^*$ recovers the distribution of $W \mid (X,Y)$ and by the hardcoded last layer, the first part of the neural network recovers an injective transformation of $U \mid (X,Y)$.

\begin{restatable}[population guarantee for SPICE-Net]{proposition}{proSPICE}
\label{prop: population guarantee SPICE-Net}
    Let $M_0$ be the true PC-SCM from \zcref{setting: pcscm}. Assume that the model for $W \mid (X,Y)$ is correctly specified, that is, there exists a $g_{\gamma} \in \mathcal{G}$ such that for a multivariate standard Gaussian random variable $\mathscr{E}$
    and
    for all $(x,y) \in \mathcal{X} \times \mathcal{Y}$, it holds that $g_{\gamma} (x, y, \mathscr{E}, E)~\sim~P^{M_0}_{W \mid (X,Y) = (x,y)}$. Furthermore, assume that the expectation in \zcref{eq: population version SPICE-Net} is finite. \newline
    Then, for the solution $g^*_\gamma$ in \zcref{eq: population version SPICE-Net}, we have for almost all $(x,y) \in \mathcal{X} \times \mathcal{Y}$ that
    \begin{align}
        g^*_\gamma (x,y,\mathscr{E},E)~\sim~P^{M_0}_{W \mid (X,Y) = (x,y)}.
    \end{align}
    Moreover, by SPICE \zcref{assumption: SPICE}-\zcref{as: continuous confounder and proxy, as: additive error, as: positive density, as: fourier transform} we have for almost all $(x,y) \in \mathcal{X} \times \mathcal{Y}$ that
    \begin{align}
        g^*_\gamma (x,y,\mathscr{E},E) - E ~\sim~P^{M_0}_{b(U) \mid (X,Y) = (x,y)}.
    \end{align}
\end{restatable}

The proof follows from \citet{Shen2023,Szekely2002} and is presented in \zcref{app:confounding equivalence}. By \zcref{prop: population guarantee SPICE-Net}, we can sample from the distribution of the unobserved confounder up to an injective transformation $b: \mathbb R^k \rightarrow \mathbb R^d$ from SPICE \zcref{assumption: SPICE}. We show that this is sufficient to recover the causal function in \zcref{prop: equivalence U and AU adjusting}, which extends the adjustment formula \citep[Theorem 3.3.2]{Pearl2009} to injective transformations of the confounder and is proven in \zcref{app:confounding equivalence}.

\begin{restatable}[adjusting for $b(U)$]{proposition}{proadjustingau}
\label{prop: equivalence U and AU adjusting}
    Let $M_0$ be the true PC-SCM from \zcref{setting: pcscm}. For all Borel measurable and injective functions $b: \mathbb{R}^k \rightarrow \mathbb{R}^d$, we recover the causal function by adjusting for $b(U)$, that is, for almost all $x \in \mathcal{X}$ we have that 
    \begin{align}
    \theta^{M_0}(x) = \mathbb{E}^{M_0} \left[ \mathbb{E}^{M_0} \left[ Y \mid b(U), X=x\right]\right].
    \end{align}
\end{restatable}

In our \zcref{setting: pcscm}, we observe $n$ independent and identically distributed copies of the random variable $(X,Y,W)$. For a Borel measurable and injective function $b: \mathbb R^k \rightarrow \mathbb R^d$ and all $i \in [n]$, we sample $b(U_i) \sim P^{M_0}_{b(U) \mid (X,Y) = (X_i,Y_i)}$ via the function $g^*_\gamma-E$ in \zcref{prop: population guarantee SPICE-Net}. Adjusting for $b(U)$ recovers the causal function by \zcref{prop: equivalence U and AU adjusting}.

In the second step of SPICE-Net, we estimate the causal function by adjusting for an injective transformation of the confounder. There are various approaches to adjust for observed confounding \citep{Yao2021} and we focus on a regression adjustment estimator following \citet{Zhang2025} based on the G-computation estimator by \citet{Robins1986,Gill2001}. Regression adjustment estimates the causal function for $b: \mathbb R^k \rightarrow \mathbb R^d$ and for all $x \in \mathcal{X}$ by
\begin{align}
    \hat{\theta}(x) := \frac{1}{n} \sum_{i=1}^n \mathbb{E}^{M_0} \left[ Y \mid b(U)=b(U_i), X=x\right].
\label{eq: regression adj}
\end{align}
We use a neural network for estimating the conditional expectation in \zcref{eq: regression adj} with one hidden layer, ReLU activations, an adaptive learning rate and squared error loss. This step of SPICE-Net can be replaced with any other causal function estimator that adjusts for observed confounding.

In the first step of SPICE-Net, we add independent samples from the distribution of $E$ in the second-to-last layer of the neural network in \zcref{fig: spicenet}. If the parameters of the distribution of $E$ are unknown, we can estimate them as additional parameters within the first step. We initialize the parameters of the distribution of $E$ by using the fact that the variance of $E$ does not exceed the variance of $W$. We keep the second step unchanged and refer to this adjusted framework as $\text{SPICE-Net-Approx}$.

\section{Application}
\label{se: application}

We compare SPICE-Net to other causal function estimators on simulated data in \zcref{subse: simulations} and on real-world data in \zcref{subse: chambers}. SPICE-Net is a two-step procedure for a Borel measurable and injective function $b: \mathbb R^k \rightarrow \mathbb R^d$:
\begin{enumerate}
    \item[1.] Train the neural network in \zcref{fig: spicenet} and sample from the conditional distribution $b(U) \mid (X,Y)$.
    \item[2.] Estimate the causal function by the regression adjustment estimator in \zcref{eq: regression adj}.
\end{enumerate}
The approach is compatible with any other causal function estimator in the second step and is implemented in the Python package \href{https://pypi.org/project/spiceaisle/}{\texttt{spiceaisle}}.

\subsection{Simulation Study}
\label{subse: simulations}

\paragraph{Simulation settings}
In our simulation, we consider four different data sets outlined in \zcref{tab: data models} which are subsumed under SPICE \zcref{assumption: SPICE}. Data set A is a linear Gaussian SCM as in \citet{Pearl2010,Kuroki2014} and data set B uses a binary treatment. Data set C deviates from Gaussian noise distributions and from a linear function $f_Y$. Data set D assumes a bivariate unobserved confounder and a three-dimensional proxy variable observed under multivariate Gaussian measurement noise with a fixed covariance matrix in which the proxies are not conditionally independent given the confounder. For all data sets, we sample $2000$ and $5000$ observations $20$ times, which we refer to as the training sets.

\begin{table}[!htbp]
\centering
\resizebox{\linewidth}{!}{%
\begin{tabular}{lllll}
        \toprule
         & $N_U$ & $f_W$ & $f_X$ & $f_Y$ \\
        \midrule
        A. Gaussian & $\mathcal{N}(0,1)$ & $U + \mathcal{N}(0,1)$& $U + \mathcal{N}(0,1)$& $U + X + \mathcal{N}(0,1)$\\
        B. Binary & $\mathcal{N}(0,1)$ & $U + \mathcal{N}(0,1)$& $\mathbbm{1}\{ \left(U + \mathcal{N}(0,1) \right) >0\}$& $U + X + \mathcal{N}(0,1)$ \\
        C. Exponential & $\text{Exp}(1)$& $U + \text{Exp}(1)$ &$U + \text{Exp}(1)$ & $X^2 + UX + \text{Exp}(1)$\\
        D. Multi-dim. & $\mathcal{N}\left(\begin{bmatrix} 0 \\0 \end{bmatrix},  \begin{bmatrix}
            1 & 0 \\ 0 & 1
        \end{bmatrix}\right)$ & $AU + E$ & $BU + \mathcal{N}(0, 1)$& $BU + X + \mathcal{N}(0, 1)$\\ 
        \bottomrule
        \vspace{0.5ex}
\end{tabular}}
      \vspace{0.5ex}
{with 
        $A =
\begin{bmatrix}
1 & 1 \\
0 & 1 \\
1 & 0
\end{bmatrix},\quad E \sim \mathcal{N}\left(
    \begin{bmatrix} 0 \\0 \\0 \end{bmatrix},
    \begin{bmatrix}
      2.0 & -0.3 & 0.5 \\
      -0.3 & 1.5 & 0.4 \\
      0.5 & 0.4 & 1.8
    \end{bmatrix}
\right) \quad\text{ and }\quad B = \begin{bmatrix}
1 \quad
1
\end{bmatrix}.$ \par}
      \caption{The data sets for our simulation that are subsumed under SPICE \zcref{assumption: SPICE}. All noise distributions are drawn independently from each other. In data set B, the treatment is one if $U + \mathcal{N}(0,1)$ is positive and zero otherwise.}
\label{tab: data models}
\end{table}

\paragraph{Causal function}
In the data sets A, B and D, the causal function for all $x \in \mathcal{X}$ is $\theta^{M_0}(x) = x$, and in data set C, it is for all $x \in \mathcal{X}$ given by $\theta^{M_0}(x) = x^2 + x +1$. We test the estimates on a test data set of size $500$ following the same data generating mechanism as the training set. For a continuous treatment, we calculate the mean squared test error as $\frac{1}{500} \sum_{i=1}^{500} (\hat{\theta}(x_i) - \theta^{M_0}(x_i))^2$ where $\hat{\theta}(x_i)$ is an estimate of the causal function evaluated at a test point. For the binary treatment in data set B, we evaluate the squared test error of estimating the ACE, which is $( \hat{\theta}(1) - \hat{\theta}(0) - 1)^2$.

\paragraph{Methods}
We outline different methods to estimate the causal function with details given in \zcref{app:simulation}. For all methods, we standardise the data and then back-transform the causal function estimates to the original scale.\footnote{We do not standardise the binary treatment in data set B. Further, we do not standardise the variables from data set B for SKPV and SPMMR because the standardisation significantly increases test errors.} For our proposed method, SPICE-Net, we keep the architecture fixed across simulations. We only adapt the sampling procedure in the second-to-last layer of \zcref{fig: spicenet} to match the distribution of $E$. Assuming that the parameters of the distribution of $E$ are unknown a priori, we estimate them with SPICE-Net-Approx. To estimate the causal function in the second step, we use the nonparametric regression adjustment estimator in \zcref{eq: regression adj}. We compare SPICE-Net and SPICE-Net-Approx to adjusting for the proxy, denoted by Adj.-W, and the (true yet in practice considered unobserved) confounder, denoted by Adj.-U, which serves as an oracle benchmark. We use the same regression adjustment estimator in \zcref{eq: regression adj} for Adj.-W and Adj.-U as for SPICE-Net and SPICE-Net-Approx.

We further compare SPICE-Net to methods with the same graph as in \zcref{setting: pcscm} such as the variational autoencoder by \citet{Louizos2017}. Their proposed method, the Causal Effect Variational Autoencoder (CEVAE), has been widely used and motivated several extensions \citep{Li2024}. CEVAE uses a factorisation of the joint density $p^M$ of \zcref{setting: pcscm} into marginal and conditional densities, which are for continuous variables approximated by Gaussian densities with parameters learnt by neural networks. 

Finally, we compare SPICE-Net to two kernel-based methods, Single Kernel Proxy Variable (SKPV) and Single Proxy Maximum Moment Restriction (SPMMR) by \citet{Xu2025}.\footnote{We conduct our simulation based on the \href{https://github.com/liyuan9988/KernelSingleProxy/tree/master}{code} by \citet{Xu2025}. Aside from using a linear kernel for the binary treatment in data set B, the implementation and hyperparameters of SKPV and SPMMR remain unchanged.} \citet{Xu2025} estimate the causal function via estimating a bridge function with a two-stage regression approach for SKPV and a maximum moment restriction approach for SPMMR. \citet{Xu2025} show that the bridge function exists if the outcome is a deterministic function of the treatment and the unobserved confounder. 

\paragraph{Results}
The test errors are presented in \zcref{plot:results_2000}, with numerical values reported in \zcref{tab:results_n2000} in \zcref{app:simulation}, which also contains the results for a training sample size of $5000$. Adjusting for the confounder achieves the lowest MSE across all data sets. Adjusting for the proxy yields biased estimates of the causal function. Excluding Adj.-U, SPICE-Net attains the lowest MSE in all data sets except for data set C and a sample size of $5000$ where SPICE-Net-Approx has a lower MSE. SPICE-Net-Approx exhibits a higher variance in test errors than SPICE-Net. In general, CEVAE, SKPV and SPMMR have a higher MSE than SPICE-Net, with a smaller difference for binary treatments but especially for non-Gaussian and multi-dimensional data.

\begin{figure}[!htbp]
  \centering
    \begin{minipage}{\linewidth}
      \centering
      \begin{minipage}[b]{0.48\linewidth}
        \centering
        \includegraphics[width=\linewidth]{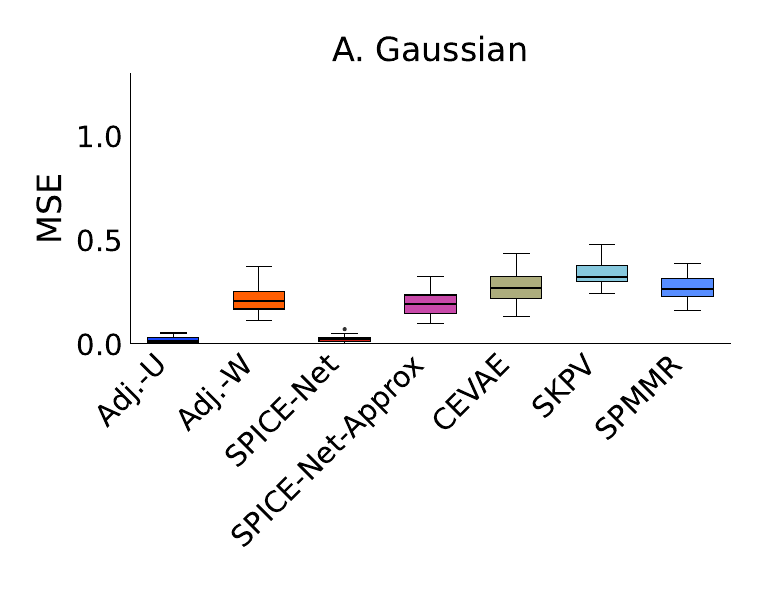}
      \end{minipage}
      \hfill
      \begin{minipage}[b]{0.48\linewidth}
        \centering
        \includegraphics[width=\linewidth]{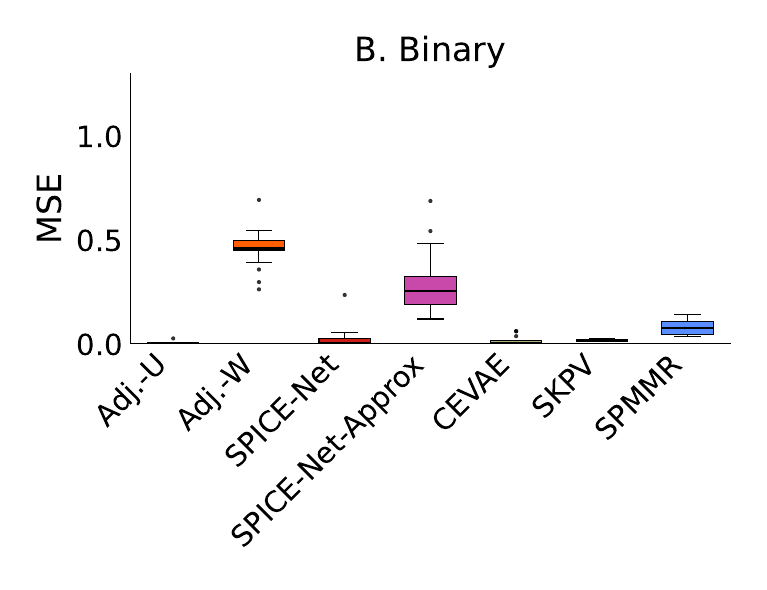}
        
      \end{minipage}

      \vspace{1ex}

      \begin{minipage}[b]{0.48\linewidth}
        \centering
        \includegraphics[width=\linewidth]{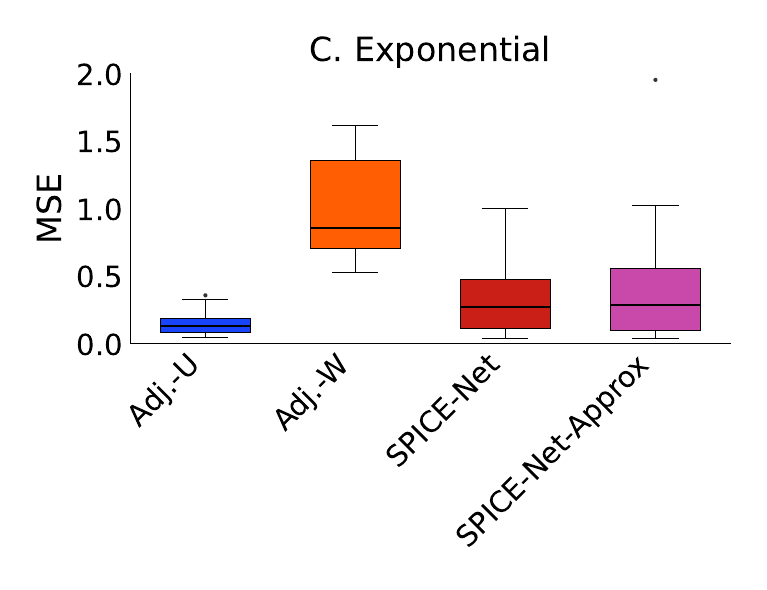}
        
      \end{minipage}
      \hfill
      \begin{minipage}[b]{0.48\linewidth}
        \centering
        \includegraphics[width=\linewidth]{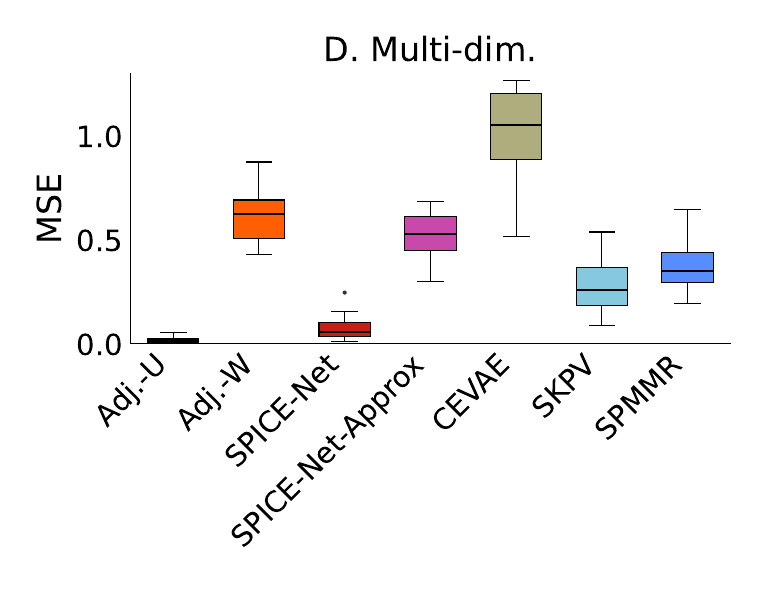}
        
      \end{minipage}

      \caption{Mean squared error (MSE) of causal function estimators described in \zcref{subse: simulations,app:simulation} for $2000$ training samples from data sets A-D of \zcref{tab: data models} on a test set of size $500$.
      For data set C, the median (standard deviation) of the MSE for CEVAE, SKPV and SPMMR are $2.87~(4.70)$, $16.16~(14.95)$ and $58.18~(18.21)$, respectively, and there is a different scale for the y-axis. Numerical values are given in \zcref{tab:results_n2000}. Our proposed method SPICE-Net achieves an MSE close to the ground-truth method Adj.-U.}
      \label{plot:results_2000}
    \end{minipage}
\end{figure}

Adjusting for the confounder satisfies the adjustment formula \citep[Theorem 3.3.2]{Pearl2009} and the low MSE indicates that the regression adjustment estimator in \zcref{eq: regression adj} performs reasonably in our simulations. In general, adjusting for the proxy is insufficient. SPICE-Net recovers the causal function in our simulations, its performance is comparable to adjusting for the confounder itself. When the parameters of the noise distribution are unknown, SPICE-Net-Approx yields reasonable results. However, we do not have theoretical guarantees for SPICE-Net-Approx and it is unclear whether it leads to valid adjustment. We interpret comparisons of SPICE-Net to other estimators with caution because SPICE-Net relies on different assumptions than SKPV and SPMMR. Nevertheless, in our setting SPICE-Net performs substantially better than the other single proxy estimators for continuous treatments. SPICE-Net improves upon adjusting for the proxy and achieves a performance close to confounder adjustment. We conclude that SPICE-Net can recover the causal function under SPICE \zcref{assumption: SPICE}.

\subsection{Real-world data}
\label{subse: chambers}

\paragraph{Causal Chamber\textsuperscript{\tiny®}}
We use the Causal Chamber\textsuperscript{\tiny®} \citep{Gamella2025} to generate real-world data with a graph as in \zcref{fig: scm and dag}. Specifically, we apply the Light Tunnel Mk2 from Causal Chamber\textsuperscript{\tiny®} as shown in \zcref{fig: causal chambers setup}, which is a computer-controlled device to measure variables from physical systems. The tunnel is a real, controllable and optical experiment designed for evaluating machine learning methodology. It consists of different light sources, polarisers and sensors that allow us to measure and manipulate the physical variables of the system, while providing a causal ground truth of the effects between them.\footnote{See \href{https://docs.causalchamber.ai/the-chambers/light-tunnel-mk2}{causalchamber.ai} for the complete documentation and ground-truth graph.} We consider a subset of the tunnel variables as described in \zcref{fig: causal chambers setup} and keep all other variables constant. 

\paragraph{Experiments and methods}
We generate data for two different parameter settings of the Causal Chamber\textsuperscript{\tiny®} as outlined in \zcref{tab: chambers exp details}.\footnote{The generated data sets of the experiments are available on \href{https://github.com/juangamella/causal-chamber/tree/main/datasets/lt_spice_v1}{github.com/juangamella/causal-chamber}.} In experiment I, we consider a low noise setting between the confounder and the proxy variable and we increase the noise in experiment II. We generate $5000$ training samples and $500$ test samples for each experiment $20$ times. We consider the same causal function estimators as in the previous simulation in \zcref{subse: simulations} with details on the parameters given in \zcref{app:chamber}. For our proposed method, SPICE-Net, we estimate the error mechanism in an additional experiment III on the Causal Chamber\textsuperscript{\tiny®} as shown in \zcref{tab: chambers exp details}. For this, we assume that the noise $E$ of the proxy-confounder relation follows a Gaussian distribution and estimate its variance by setting the confounder to zero. We then use the empirical variance of the proxy as an estimate for the variance of $E$. This is different from SPICE-Net-Approx, where we estimate the variance of $E$ within the estimation framework as described in \zcref{se: method}.

\paragraph{Results} 
We present the test errors of all causal function estimators compared to the ground-truth method Adj.-U in \zcref{plot:chamber_results} with numerical values in \zcref{tab:chambers_3panel_mse}. In a low noise setting, Adj.-U, Adj.-W, SPICE-Net and SPICE-Net-Approx lead to similar causal function estimates. In a noisy setting, Adj.-W and SPICE-Net-Approx have a higher MSE and deviate from Adj.-U and SPICE-Net. SPICE-Net performs better than the other proxy-based causal function estimators and achieves a performance close to confounder adjustment.

\begin{figure}[!htbp]
    \centering
    \begin{minipage}{\linewidth}
      \centering
      \begin{minipage}[b]{0.48\linewidth}
        \centering
        \includegraphics[width=\linewidth]{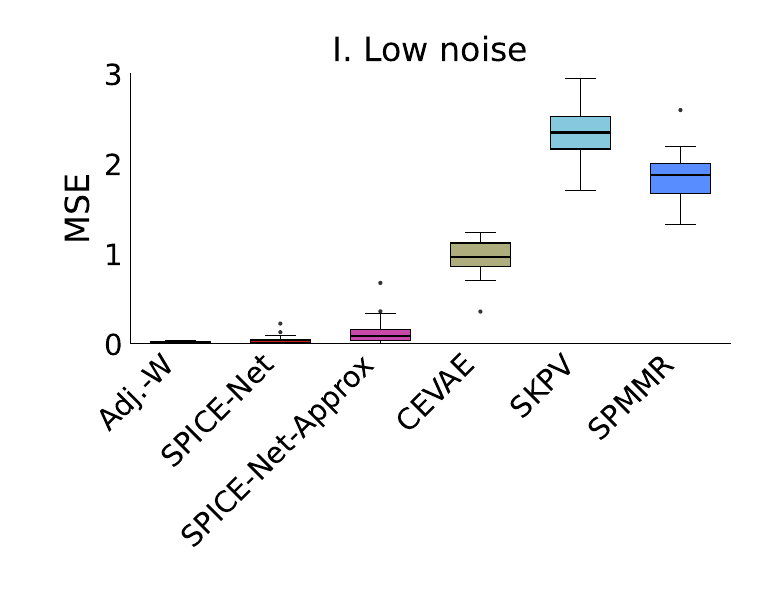}
      \end{minipage}
      \hfill
      \begin{minipage}[b]{0.48\linewidth}
        \centering
        \includegraphics[width=\linewidth]{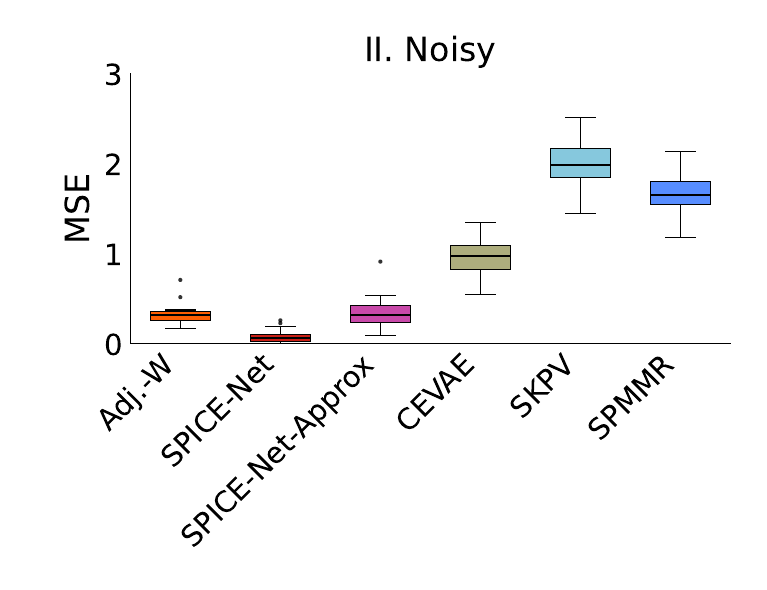}
        
      \end{minipage}
    \end{minipage}
    
    \caption{Mean squared error (MSE) in ten thousands of causal function estimators compared to the ground-truth method Adj.-U on data from the Light Tunnel Mk2 from the Causal Chamber\textsuperscript{\tiny®} \citep{Gamella2025} from experiments I and II of \zcref{tab: chambers exp details}. The methods are described in \zcref{subse: simulations,subse: chambers,app:simulation,app:chamber} and numerical values of the results are given in \zcref{tab:chambers_3panel_mse}. Our proposed method SPICE-Net is close to the ground-truth method Adj.-U in both experiments.
    }
      \label{plot:chamber_results}
\end{figure}

\section{Discussion}
\label{se: discussion}

In this paper, we prove the identifiability of the causal function in a proxy-confounded structural causal model under the assumption of a known error mechanism. The causal function is identifiable under SPICE \zcref{assumption: SPICE} for continuous confounders and proxies in an additive noise model for a broad class of noise distributions. Based on this result, we propose SPICE-Net, a machine learning method to recover the distribution of the unobserved confounder up to an injective transformation. It can be combined with any estimator of the causal function and it performs on par with confounder adjustment in our simulations and real-world example.  

\paragraph{Applications}
We expect SPICE and SPICE-Net to be relevant in medical applications. In healthcare, the functional status of senior citizens confounds the relationship between influenza vaccination and all cause mortality \citep{Jackson2005}. The diagnosis of dementia is an indicator of functional status and may be associated with a higher mortality risk and a lower
probability of influenza vaccination \citep{Jackson2005}. In administrative insurance data sets, Alzheimer's disease and related disorders are often unmeasured \citep{Wyss2022}. However, the use of donepezil, a medication to treat dementia, is typically recorded and can serve as a proxy variable. In cancer research, a patient’s overall fitness can be an unobserved confounder influencing both the choice of treatment aimed at curing cancer and patient survival. Such a scenario arises in stage III non-small cell lung cancer, where \citet{vanAmsterdam2022} examine concurrent versus sequential chemoradiation treatment. A patient’s performance score, reflecting their capacity for self-care and physical ability, serves as a proxy for overall fitness in this context. An additional example comes from educational contexts akin to \citet{Deaner2023}. After elementary school in Germany, a decision on the secondary school track is made. Treating track choice as the treatment and degree completion as the outcome, their relation is confounded by the academic ability of the student. Since ability is unobserved, grades from elementary school serve as error-prone proxy measurements. In general, our setting is related to measurement errors in which the proxy is a noisy measurement of the confounder. In particular, it corresponds to a nondifferential measurement error, where the proxy is conditionally independent of the outcome given the treatment and the confounder \citep{Carroll2006}.

\paragraph{Extensions}
We sketch three generalisations to the PC-SCM of \zcref{setting: pcscm} under which our results remain valid. First, we allow for an additional observed confounder $O \in \mathcal{O} \subseteq \mathbb{R}^l$, where $l \in \mathbb{N}$, as illustrated on the left-hand side of \zcref{fig: extensions pcscm}. We show in \zcref{app:generalisations} that the causal function in the extended setting remains identifiable if the error mechanism is complete. Consequently, the SPICE \zcref{assumption: SPICE} remains sufficient for identifying the causal function.

Second, we can reverse the edge between the treatment and the confounder as illustrated by the undirected edge in the middle graph of \zcref{fig: extensions pcscm}. Here, we can reinterpret our PC-SCM within the context of unobserved mediation analysis by considering a model similar to \zcref{setting: pcscm} but where $U\leftarrow f_U(X,N_U)$ and $X\leftarrow N_X$ such that the induced graph is as in \zcref{setting: pcscm} but with the edge reversed between $U$ and $X$. In this setting, the causal function corresponds to the controlled direct effect of $X$ on $Y$ as in \citet{Liu2025}. The direct effect can relate to implicit biases in clinical decision-making as opposed to biological effects mediated through the patient’s true underlying health status \citep{Liu2025}. Another example is to assess whether genetic variants influence lung cancer directly, or indirectly via smoking behaviour \citep{Vanderweele2012}. The unobserved mediation graph is termed classical measurement error by \citet{LeCessie2012}. It is observationally equivalent to the PC-SCM by \citet[Theorem 1.2.8]{Pearl2009} because they have the same skeleton and the same (empty) set of v-structures and hence our results remain valid. 

Lastly, we consider a case where all variables are measured with error as on the right-hand side of \zcref{fig: extensions pcscm}. We assume that we know the error mechanisms for $W$ and $U$, but also for the relation between $X^*$ and $X$ and $Y^*$ and $Y$. If all three error mechanisms are complete, the causal function of \zcref{def: causal function and ACE}, which considers the unobserved treatment and outcome, is identifiable as we show in \zcref{app:generalisations}.

\begin{figure}[!tbp]
\centering
    \begin{minipage}{0.31\linewidth}
    \centering
    \begin{tikzpicture}[scale=0.45, >=stealth, node distance=2cm]

        \node (U) at (0,0) {$U$};
        \node (W) at (0,2) {$W$};
        \node (X) at (-2,-2) {$X$};
        \node (Y) at (2,-2) {$Y$};
        \node (O) at (0,-4) {$O$};

        \draw[-latex] (U) -- (W);
        \draw[-latex] (U) -- (X);
        \draw[-latex] (U) -- (Y);
        \draw[-latex] (X) -- (Y);
        \draw[-latex] (O) -- (X);
        \draw[-latex] (O) -- (Y);

        \draw[dashed] (U) circle (0.5cm);

    \end{tikzpicture}
    \end{minipage}
    \begin{minipage}{0.31\linewidth}
    \centering
    \begin{tikzpicture}[scale=0.45, >=stealth, node distance=2cm]

        \node (U) at (0,0) {$U$};
        \node (W) at (0,2) {$W$};
        \node (X) at (-2,-2) {$X$};
        \node (Y) at (2,-2) {$Y$};

        \draw[-latex] (U) -- (W) node[midway, right] {};
        \draw  (X) -- (U) node[midway, left] {};
        \draw[-latex] (U) -- (Y) node[midway, right] {};
        \draw[-latex] (X) -- (Y) node[midway, above] {};

        \draw[dashed] (U) circle (0.5cm);

    \end{tikzpicture}
    \end{minipage}
    \begin{minipage}{0.31\linewidth}
    \centering
    \begin{tikzpicture}[scale=0.45, >=stealth, node distance=2cm]

        \node (U) at (0,0) {$U$};
        \node (W) at (0,2) {$W$};
        \node (X) at (-2,-2) {$X$};
        \node (Y) at (2,-2) {$Y$};
        \node (X_s) at (-2,0) {$X^*$};
        \node (Y_s) at (2,0) {$Y^*$};

        \draw[-latex] (U) -- (W) node[midway, right] {};
        \draw[-latex] (U) -- (X) node[midway, left] {};
        \draw[-latex] (U) -- (Y) node[midway, right] {};
        \draw[-latex] (X) -- (Y) node[midway, above] {};
        \draw[-latex] (X) -- (Y) node[midway, above] {};
        \draw[-latex] (X) -- (X_s) node[midway, above] {};
        \draw[-latex] (Y) -- (Y_s) node[midway, above] {};
        \draw[dashed] (U) circle (0.5cm);
        \draw[dashed] (X) circle (0.5cm);
        \draw[dashed] (Y) circle (0.5cm);

    \end{tikzpicture}
    \end{minipage}
    
    \caption{Extensions of \zcref{setting: pcscm} with an additional observed confounder $O \in \mathcal{O} \subseteq\mathbb{R}^l$ (left), unobserved mediation (middle) and noisy treatment and outcome (right).}
    \label{fig: extensions pcscm}
\end{figure}
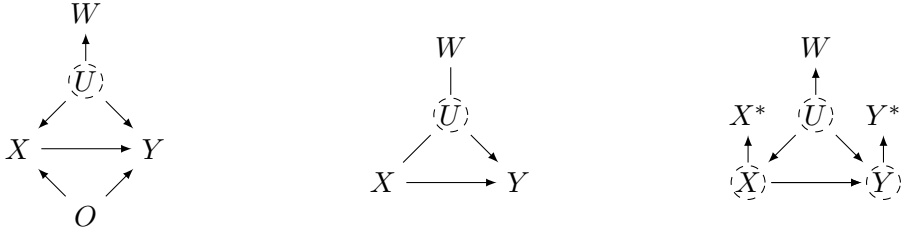

\paragraph{Known and Unknown Error Mechanisms}
In this paper, we assume that we know the error mechanism, that is the set of conditional densities of the proxy given the confounder. In the case of an additive noise model as in SPICE \zcref{assumption: SPICE}, this corresponds to the density of the noise $E$. We provide an example from differential privacy in which these assumptions are met. The United States census records demographic and housing characteristics of individuals in the United States and is recognized as the world’s longest-running statistical program \citep{Garfinkel2022}. To protect respondent privacy, the Census Bureau implemented a differential privacy mechanism by adding known, discrete Gaussian noise to the data \citep{Census2021}. If the error mechanism is generally unknown, but we have access to noise-free measurements of the confounder for a subset of the data, we can calibrate our causal function estimates for the entire population accordingly \citep{Pearl2010,Kuroki2014}. If we only know the distribution of the noise in the additive noise model, we can estimate its parameters using SPICE-Net-Approx. However, we lack theoretical guarantees for SPICE-Net-Approx, and establishing them is an avenue for future work. When the error mechanism is only partially observable, a natural direction for future research is to derive bounds on the causal function as in \citet{Zhang2024}. Moreover, our work may provide additional information about the assumptions in settings with multiple proxies, such as for \citet{Deaner2023}, where identifiability is obtained using three proxies without the knowledge of the error mechanism. 

Future work could derive other conditions for completeness of the error mechanism. Moreover, there appears to be a trade-off between assumptions on the noise density and assumptions on the conditional densities of the confounder given treatment and outcome. In our implementation, SPICE-Net uses the energy loss as in Engression \citep{Shen2023}, although any strictly proper scoring rule is applicable. Overall, we think that a key next step for proxy-based methods is practical evaluation, for example in medical applications and we hope that SPICE will motivate future work on this.
\acks{We thank Simon Bing, Simon Buchholz, Hubert Drazkowski, Francesco Freni, Juan Gamella, Luigi Gresele, Niels Richard Hansen, Harald Kugler, Kai Lion, Anton Rask Lundborg and Xinwei Shen.
\newline
The authors used large language models to assist with proofreading the manuscript while retaining full responsibility for its content.
\newline
This research was supported by a research grant (0069071) from Novo Nordisk Fonden and the Pioneer Centre for AI (DNRF grant number P1).
}

\newpage
\appendix

\section{Ordinary Least Squares Adjustment}
\label{app:linear_Gaussian_setting}

Here, we compare adjustment for the unobserved confounder $U$ with adjustment for the proxy variable $W$. We consider a univariate linear Gaussian SCM and use ordinary least squares to estimate the causal function. In a linear Gaussian model, the functions $f_W, f_X, f_Y$ are linear such that 
\begin{equation}
    \label{SCM linear 1}
        M^{\text{lin}}:\left\{
        \begin{array}{rl}
        U & \leftarrow N_U \\
        W & \leftarrow \alpha_{UW} U + E  \\
        X & \leftarrow \alpha_{UX} U + N_X \\
        Y & \leftarrow \alpha_{UY} U + \alpha_{XY} X + N_Y
        \end{array}
        \right.
\end{equation}
with jointly independent mean-zero Gaussian variables $(N_U, E, N_X, N_Y)$ with variances $\sigma_{\cdot}^2 \in \mathbb{R}_{+}$ and path coefficients $\alpha_{\cdot \cdot} \in \mathbb{R}$. This means that the support of all random variables is $\mathcal{U} = \mathcal{W} = \mathcal{X} = \mathcal{Y} =\mathbb{R}$. The joint distribution under causal model $M^{\text{lin}}_0$ is
    \begin{align}
    \label{joint normal distribution 1}
        & \begin{bmatrix}
        U \\
        W \\
        X \\
        Y
        \end{bmatrix}
        \sim \mathcal{N}\left(\begin{bmatrix}
            0 \\ 0 \\ 0 \\ 0
        \end{bmatrix}, \mathbf{V}^{M^{\text{lin}}_0}\right), \quad \text{with} \\
        & \mathbf{V}^{M^{\text{lin}}_0} = \scalebox{0.55}{$
        \begin{bmatrix}
        \sigma_{N_{U}}^2 & \alpha_{UW}\sigma_{N_{U}}^2 & \alpha_{UX}\sigma_{N_{U}}^2 & (\alpha_{UY} + \alpha_{XY}\alpha_{UX})\sigma_{N_{U}}^2 \\
        \alpha_{UW}\sigma_{N_{U}}^2 & \alpha_{UW}^2\sigma_{N_{U}}^2 + \sigma_{E}^2 & \alpha_{UW}\alpha_{UX}\sigma_{N_{U}}^2 & (\alpha_{UY} + \alpha_{XY}\alpha_{UX})\alpha_{UW}\sigma_{N_{U}}^2 \\
        \alpha_{UX}\sigma_{N_{U}}^2 & \alpha_{UX}\alpha_{UW}\sigma_{N_{U}}^2 & \alpha_{UX}^2\sigma_{N_{U}}^2 + \sigma_{N_{X}}^2 & (\alpha_{UY} + \alpha_{XY}\alpha_{UX})\alpha_{UX}\sigma_{N_{U}}^2 + \alpha_{XY}\sigma_{N_{X}}^2 \\
        (\alpha_{UY} + \alpha_{XY}\alpha_{UX})\sigma_{N_{U}}^2 & (\alpha_{UY} + \alpha_{XY}\alpha_{UX})\alpha_{UW}\sigma_{N_{U}}^2 & (\alpha_{UY} + \alpha_{XY}\alpha_{UX})\alpha_{UX}\sigma_{N_{U}}^2 + \alpha_{XY}\sigma_{N_{X}}^2 & (\alpha_{UY} + \alpha_{XY}\alpha_{UX})^2\sigma_{N_{U}}^2 + \alpha_{XY}^2\sigma_{N_{X}}^2 + \sigma_{N_{Y}}^2
        \end{bmatrix}$.} 
    \end{align}
We omit the superscript $M^{\text{lin}}_0$ in the following for clarity. We have for all $x \in \mathcal{X}$ that
    \begin{align}
        \theta (x) &= \mathbb{E}^{do (X :=x)}\left[Y \right] \\
        &= \mathbb{E}\left[ \alpha_{UY} U + \alpha_{XY} x + N_Y \right]\\
        &= \alpha_{XY} x.
    \end{align}
This holds because the random noise variables have mean zero. We denote covariances by $\sigma$, for example the covariance between $U$ and $X$ is denoted by $\sigma_{U,X}$. The population solution of the ordinary least squares estimator for regressing $Y$ on $X$ and $U$ is given by
    \begin{align}
        \begin{bmatrix}
            \sigma_{X,X} & \sigma_{U,X} \\
            \sigma_{U,X} & \sigma_{U,U}
        \end{bmatrix}^{-1} 
        \begin{bmatrix}
            \sigma_{X,Y} \\
            \sigma_{U,Y}
        \end{bmatrix}
        &=
        \begin{bmatrix}
            \frac{\sigma_{X,Y}\sigma_{U,U} - \sigma_{U,Y} \sigma_{U,X}}{\sigma_{X,X} \sigma_{U,U} - (\sigma_{U,X})^2} \\
            \frac{\sigma_{U,Y} \sigma_{X,X} - \sigma_{X,Y} \sigma_{U,X}}{\sigma_{X,X} \sigma_{U,U} - (\sigma_{U,X})^2}
        \end{bmatrix}.
    \end{align}
We reformulate the regression coefficient for $X$ by inserting the corresponding covariances from \zcref{joint normal distribution 1} such that for all $x\in\mathcal{X}\setminus\{0\}$
\begin{align}
        & \frac{\sigma_{X,Y}\sigma_{U,U} - \sigma_{U,Y} \sigma_{U,X}}{\sigma_{X,X} \sigma_{U,U} - (\sigma_{U,X})^2} \\
        &= \frac{((\alpha_{UY} + \alpha_{XY}\alpha_{UX})\alpha_{UX}\sigma_{N_{U}}^2 + \alpha_{XY}\sigma_{N_{X}}^2) \sigma_{N_{U}}^2  - (\alpha_{UY} + \alpha_{XY}\alpha_{UX})\sigma_{N_{U}}^2 \alpha_{UX}\sigma_{N_{U}}^2}{(\alpha_{UX}^2\sigma_{N_{U}}^2 + \sigma_{N_{X}}^2) \sigma_{N_{U}}^2 -  (\alpha_{UX}\sigma_{N_{U}}^2)^2} \\
         &= \frac{\alpha_{UY} \alpha_{UX} \sigma_{N_{U}}^2 + \alpha_{XY} \alpha_{UX}^2 \sigma_{N_{U}}^2 + \alpha_{XY} \sigma_{N_{X}}^2 - \alpha_{UY} \alpha_{UX} \sigma_{N_{U}}^2 - \alpha_{XY} \alpha_{UX}^2 \sigma_{N_{U}}^2}{\alpha_{UX}^2 \sigma_{N_{U}}^2 + \sigma_{N_{X}}^2 -  \alpha_{UX}^2 \sigma_{N_{U}}^2} \\
         &= \alpha_{XY} = \frac{\theta(x)}{x}. \label{eq: reg on X and U}
\end{align}
If we consider a regression of $Y$ on $W$ and $X$, the regression coefficient for $X$ is
\begin{align}
    &\frac{\sigma_{X,Y}\sigma_{W,W} - \sigma_{W,Y} \sigma_{W,X}}{\sigma_{X,X} \sigma_{W,W} - (\sigma_{W,X})^2} \\
    &=\frac{\sigma_{W,W}\big(\alpha_{XY}\sigma_{X,X}+\alpha_{UX}\alpha_{UY}\sigma_{U,U}\big)
     -\sigma_{W,X}\big(\alpha_{XY}\sigma_{W,X}+\alpha_{UW}\alpha_{UY}\sigma_{U,U}\big)}{\sigma_{X,X} \sigma_{W,W} - (\sigma_{W,X})^2} \\
     &=\alpha_{XY} + \frac{\sigma_{W,W}\alpha_{UX}\alpha_{UY}\sigma_{U,U}
     -\sigma_{W,X}\alpha_{UW}\alpha_{UY}\sigma_{U,U}}{\sigma_{X,X} \sigma_{W,W} - (\sigma_{W,X})^2} \\
    &=\alpha_{XY} +  \frac{ 
    \alpha_{UX} \alpha_{UY} \sigma_{N_U} ^2\sigma_{E}^2 }{\sigma_{X,X} \sigma_{W,W} - (\sigma_{W,X})^2} \\
    &=\alpha_{XY} +  \frac{ 
    \alpha_{UX} \alpha_{UY} \sigma_{N_U} ^2\sigma_{E}^2 }{\alpha_{UX}^2 \sigma_{N_U}^2 \sigma_{E}^2 +  \alpha_{UW}^2 \sigma_{N_X}^2 \sigma_{N_U}^2 + \sigma_{N_X}^2 \sigma_{E}^2} \\
    &=\alpha_{XY} +  \frac{ 
    \alpha_{UX} \alpha_{UY} }{\alpha_{UX}^2 + \frac{\alpha_{UW}^2 \sigma_{N_X}^2}{\sigma_{E}^2} + \frac{\sigma_{N_X}^2}{\sigma_{N_U}^2}}. \label{eq: reg on X and W}
\end{align}
The regression coefficient in \zcref{eq: reg on X and W} deviates from \zcref{eq: reg on X and U} if we have confounding, meaning that $\alpha_{UX} \neq 0$ and $\alpha_{UY} \neq 0$. The deviation reduces if the proxy is close to $U$ (meaning that $\sigma_{E}^2$ is small or $\alpha_{UW}$ is large) or if the direct effect of $X$ on $Y$ dominates the confounding effect due to $U$ (meaning that $\alpha_{XY}$ is large, $\sigma_{N_X}^2$ is large or $\sigma_{N_U}^2$ is small). These relationships are described holding all other parameters constant.

\section{A Complete Error Mechanism Implies Identifiability}
\label{app:proof_theorem}

\thmidentifiability*
\begin{proof}[of \zcref{thm: effect identifiable}]
\label{proof: theorem identifiability}
    For almost all $(w,x,y) \in \mathcal{W} \times \mathcal{X} \times \mathcal{Y}$, we have
    \begin{align}
            &p^{M_0}_{W \mid X,Y} (w \mid x,y) \\& =\int  p^{M_0}_{U,W \mid X,Y} (u,w \mid x,y) \, \mu(du) \\
            & =\int p^{M_0}_{W \mid U,X,Y} (w \mid u,x,y) p^{M_0}_{U \mid X,Y} (u \mid x, y) \, \mu(du) \\
            & =\int  p^{M_0}_{W \mid U} (w \mid u) p^{M_0}_{U \mid X,Y} (u \mid x, y) \, \mu(du)    \label{eq: proof thm identifiability 1}
    \end{align}    
    where the third equation follows from the independence $W \indep (X, Y) \mid U$. Furthermore, by the definition of $\mathcal{M}(M_0)$, it holds for all $M \in \mathcal{M} (M_0)$ that for almost all $(u, x, y) \in \mathcal{U} \times \mathcal{X} \times \mathcal{Y}$
    \begin{equation}
    \label{eq: proof thm identifiability 2}
    p^{M}_{W \mid U} (\cdot \mid u) = p^{M_0}_{W \mid U }(\cdot \mid u)  \text{ and } p^{M}_{ W \mid X,Y}(\cdot \mid x,y)  = p^{M_0}_{W \mid X,Y} (\cdot \mid x,y).
    \end{equation} 
    This implies for all $M \in \mathcal{M} (M_0)$ and for almost all $(w,x,y) \in \mathcal{W} \times \mathcal{X} \times \mathcal{Y}$ that
    \begin{align}
        \label{eq: proof thm identifiability 3.1}
        \int p^{M_0}_{W \mid U} (w \mid u) p^{M}_{U \mid X,Y} (u \mid x, y) \, \mu(du) = \int p^{M_0}_{W \mid U} (w \mid u) p^{M_0}_{U \mid X,Y} (u \mid x, y) \, \mu(du)
    \end{align}
    and hence
    \begin{align}
        \label{eq: proof thm identifiability 3.2}
        \int p^{M_0}_{W \mid U} (w \mid u) \left( p^{M}_{U \mid X,Y}(u  \mid x, y) - p^{M_0}_{U \mid X,Y}(u  \mid x, y)\right) \, \mu(du) = 0
    \end{align}
    where we apply \zcref{eq: proof thm identifiability 1} and \zcref{eq: proof thm identifiability 2}. Considering the difference in \zcref{eq: proof thm identifiability 3.2}, we have for all $M \in \mathcal{M} (M_0)$ and for almost all $(x, y) \in \mathcal{X} \times \mathcal{Y}$ that
    \begin{align}
    \label{eq: proof thm identifiability 4}
        &\int \left |p^{M}_{U \mid X,Y} (u \mid x, y) - p^{M_0}_{U \mid X,Y} (u \mid x, y) \right | \, \mu(du)\\
        & \leq \int \left |p^{M}_{U \mid X,Y} (u \mid x, y)\right | \, \mu(du) + \int \left |p^{M_0}_{U \mid X,Y} (u \mid x, y) \right | \, \mu(du)\\
        & = \int p^{M}_{U \mid X,Y} (u \mid x, y) \, \mu(du) + \int p^{M_0}_{U \mid X,Y} (u \mid x, y)  \, \mu(du) \\
        & = 2.
    \end{align}
    Hence for all $M \in \mathcal{M} (M_0)$ and for almost all $(x, y) \in \mathcal{X} \times \mathcal{Y}$ it holds that $p^{M}_{U \mid X,Y}(\cdot \mid x, y) - p^{M_0}_{U \mid X,Y}(\cdot \mid x, y)\in L_1(\mathcal{U})$. 
    We have for all $M \in \mathcal{M} (M_0)$, using \zcref{def: complete error} and \zcref{eq: proof thm identifiability 3.2}, and for almost all $(u, x, y) \in \mathcal{U} \times \mathcal{X} \times \mathcal{Y}$ that
    \begin{align}
    \label{eq:dens_conf_trt_outcome_same}
        p^{M}_{U \mid X,Y}(u \mid x, y) - p^{M_0}_{U \mid X,Y}(u \mid x, y) = 0.
    \end{align}    
    The causal function from \zcref{def: causal function and ACE} for almost all $x \in \mathcal{X}$ is given by
    \begin{align}
        \theta^{M_0} (x) &=  \int y \int  p^{M_0}_{Y \mid U,X} \left(y \mid u, x \right) p^{M_0}_{U}(u) \, \mu(du, dy) \\
        &=  \int y \int   \frac{p^{M_0}_{U,X,Y}(u,x,y)}{p^{M_0}_{U,X} (u,x)}p^{M_0}_{U} (u)\, \mu(du, dy) \\
        &=    \int y \int   \frac{p^{M_0}_{U \mid X,Y} (u \mid x,y) p^{M_0}_{X,Y}(x,y)}{  \int p^{M_0}_{U \mid X,Y} (u \mid x,y) p^{M_0}_{X,Y}(x,y)\, \mu(dy) } 
        \\ & \left(\int \int p^{M_0}_{U \mid X,Y} (u \mid x,y) p^{M_0}_{X,Y} (x,y)\, \mu(dx, dy)\right) \, \mu(du, dy).    \label{eq: proof thm identifiability 6}
    \end{align}
    In \zcref{eq: proof thm identifiability 6}, the causal function only depends on observational features and on the conditional densities of the confounder given the treatment and the outcome. For all $M \in \mathcal{M} (M_0)$, the observational features coincide by \zcref{eq: obs features same}, and for the conditional densities of the confounder given the treatment and the outcome, we established equality in \zcref{eq:dens_conf_trt_outcome_same}. Thus we have for all $M \in \mathcal{M} (M_0)$ that
    \begin{equation}
    \label{eq: proof thm identifiability 7}
        \theta^M \equiv  \theta^{M_0} \quad \text{a.e.}
    \end{equation}
    Therefore, the causal function $\theta^{M_0}$ is identifiable according to \zcref{def: identifiability}.
\end{proof}

\section{A Complete Error Mechanism for Discrete Confounders and Proxies}
\label{app:proof_corollary_complete_discrete}

\thmdiscrete*
\begin{proof}[of \zcref{thm: discrete confounder}]
\label{proof: discrete confounder}
    The premise in \zcref{def: complete error} in the discrete setting of \zcref{thm: discrete confounder} is that for all $\delta \in L_1 (\mathcal{U})$, it holds that
    \begin{align}
       \left(\begin{array}{ccc}
        p^{M_0}_{W \mid U} (w_1' \mid u_1) & \cdots & p^{M_0}_{W \mid U} (w_1' \mid u_{k}) \\
        \vdots & \ddots & \vdots \\
        p^{M_0}_{W \mid U} (w_r' \mid u_1) & \cdots & p^{M_0}_{W \mid U} (w_r' \mid u_{k})
        \end{array}\right)
        \left(\begin{array}{c}
        \delta (u_1)  \\
        \vdots  \\
       \delta (u_{k})  
        \end{array}\right)
        =  
        \left(\begin{array}{c}
        0  \\
        \vdots  \\
       0  
        \end{array}\right) \quad \text{a.e.}
    \end{align}
    Indeed, for all $\delta \in L_1 (\mathcal{U})$ that satisfy this premise it follows that
    \begin{align}
        \left(\begin{array}{c}
        \delta (u_1)  \\
        \vdots  \\
       \delta (u_{k})  
        \end{array}\right)
        =  
        \left(\begin{array}{c}
        0  \\
        \vdots  \\
       0  
        \end{array}\right)
        \quad\text{a.e.}
    \end{align}
    because the matrix from \zcref{eq: error matrix} has full column rank by Assumption \zcref{as: invertibilty}. The error mechanism is $L_1$-complete.
\end{proof}

\section{Previous Identifiability Results}
\label{app:proofs_old_results}
We state and prove the results on the identifiability of the causal function of \citet{Pearl2010,Kuroki2014}.
\begin{restatable}[results on the identifiability of the causal function]{proposition}{prooldresults}
\label{prop: all identifiability results}
Let $M_0$ be the true PC-SCM from \zcref{setting: pcscm}. Assume that one of the following holds.
\begin{enumerate}[label=(\roman*)]
    \item \label{prop: discrete confounding} We have a discrete confounder $U \in \mathcal{U} = \{u_1,...,u_d\}$ and a discrete proxy $W \in  \mathcal{W} = \{w_1,...,w_d\}$ with a finite number of categories $d \in \mathbb{N}$. The error mechanism matrix 
    \begin{equation}
    \label{eq: error mechanism matrix pearl}
        \left(\begin{array}{ccc}
        p^{M_0}_{W \mid U} (w_1 \mid u_1) & \cdots & p^{M_0}_{W \mid U} (w_1 \mid u_d) \\
        \vdots & \ddots & \vdots \\
        p^{M_0}_{W \mid U} (w_d \mid u_1) & \cdots & p^{M_0}_{W \mid U} (w_d \mid u_d)
        \end{array}\right)
    \end{equation}
    is known and invertible.
    \item \label{prop: linear Gaussian} We have a linear Gaussian structural causal model where $f_W, f_X, f_Y$ are linear functions such that 
    \begin{equation}
    \label{SCM linear}
        M^{\text{lin}}:\left\{
        \begin{array}{rl}
        U & \leftarrow N_U \\
        W & \leftarrow \alpha_{UW} U + E  \\
        X & \leftarrow \alpha_{UX} U + N_X \\
        Y & \leftarrow \alpha_{UY} U + \alpha_{XY} X + N_Y
        \end{array}
        \right.
    \end{equation}
    with jointly independent mean-zero Gaussian variables $(N_U, E, N_X, N_Y)$ with variances $\sigma_{{\cdot}}^2 \in \mathbb{R}_{+}$ and path coefficients $\alpha_{\cdot \cdot} \in \mathbb{R}$. This means that the support of all the random variables is $\mathcal{U} = \mathcal{W} = \mathcal{X} = \mathcal{Y} =\mathbb{R}$. Further, knowing the error mechanism corresponds to knowing the measurement error variance ${\sigma^2}_{E}^{M_0^{lin}}$.
    
    Then, the causal function $\theta^{M_0}$ is identifiable.
\end{enumerate}
\end{restatable}
\begin{proof}[of \zcref{prop: all identifiability results}-\zcref{prop: discrete confounding}.]
   We prove the identifiability of $\theta^{M_0}$ for a discrete setting in which $(W, U) \in \{0,1\}^2$. \citet{Pearl2010,Kuroki2014} extend the setting to confounders and proxies with more than two categories. We want to show that if the matrix
    \begin{equation}
    f^{M_0} := 
        \begin{pmatrix}
        p^{M_0}_{W \mid U} (0 \mid 0) & p^{M_0}_{W \mid U} (0 \mid 1) \\
        p^{M_0}_{W \mid U} (1 \mid 0) & p^{M_0}_{W \mid U} (1 \mid 1)
        \end{pmatrix}
    \end{equation}
    is known and invertible, then $\theta^{M_0}$ is identifiable. 
    
    The interventional probability mass function for almost all $(x,y) \in \mathcal{X} \times \mathcal{Y}$ in a binary confounder model is given by
    \begin{equation}
    \label{eq: int distr binary}
        p^{M_0 ; do(X:=x)}_Y (y)= p^{M_0}_{Y \mid U,X} (y \mid 0,x) p^{M_0}_U(0) +  p^{M_0}_{Y \mid U,X} (y \mid 1,x) p^{M_0}_{U}(1).
    \end{equation}
    We follow the approach by \citet{Greenland2008,Pearl2010,Kuroki2014} known as matrix adjustment method and effect restoration. We have for almost all $(w,x,y) \in \{0,1\} \times \mathcal{X} \times \mathcal{Y}$ that
        \begin{align}
        p^{M_0}_{W,X,Y} (w,x,y) &=  p^{M_0}(0, w, x, y) + p^{M_0}(1, w, x, y)\\
        &=  p^{M_0}_{W \mid U,X,Y} (w \mid 0, x, y) p^{M_0}_{U,X,Y}(0, x, y) + p^{M_0}_{W \mid U,X,Y} (w \mid 1, x, y) p^{M_0}_{U,X,Y} (1, x, y) \\
        &=  p^{M_0}_{W \mid U} (w \mid 0) p^{M_0}_{U,X,Y}(0, x, y) + p^{M_0}_{W \mid U} (w \mid 1) p^{M_0}_{U,X,Y} (1, x, y)
        \end{align}
    where in the last equation we use \zcref{eq: independence statement}. In matrix notation, we have for almost all $(x,y) \in \mathcal{X} \times \mathcal{Y}$ that
    \begin{equation}
        \begin{pmatrix}
        p^{M_0}_{W,X,Y} (0, x, y) \\
        p^{M_0}_{W,X,Y} (1, x, y)
        \end{pmatrix}
        =
        f^{M_0}
        \begin{pmatrix}
        p^{M_0}_{U,X,Y} (0,x, y) \\
        p^{M_0}_{U,X,Y}(1,x, y)
        \end{pmatrix}.
    \end{equation}
    Since $f^{M_0}$ is invertible, we have for almost all $(x,y) \in \mathcal{X} \times \mathcal{Y}$ that
    \begin{equation}
    \label{eq: proof bianry conf}
        \begin{pmatrix}
        p^{M_0}_{U,X,Y} (0,x, y) \\
        p^{M_0}_{U,X,Y}(1,x, y)
        \end{pmatrix}
        =
        (f^{M_0})^{-1}
        \begin{pmatrix}
        p^{M_0}_{W,X,Y} (0, x, y) \\
        p^{M_0}_{W,X,Y} (1, x, y)
        \end{pmatrix}.
    \end{equation}
    For all observationally equivalent models $M \in \mathcal{M} (M_0)$, it holds that $\mathcal{F}^{M} = \mathcal{F}^{M_0}$. This implies for all $M \in \mathcal{M} (M_0)$ and for almost all $(x,y) \in \mathcal{X} \times \mathcal{Y}$ that
    \begin{equation}
        \begin{pmatrix}
        p^{M}_{U,X,Y}(0, x, y) \\
        p^{M}_{U,X,Y}(1, x, y)
        \end{pmatrix}
        =
        \begin{pmatrix}
        p^{M_0}_{U,X,Y}(0, x, y) \\
        p^{M_0}_{U,X,Y} (1, x, y)
        \end{pmatrix}.
    \end{equation}
    This implies, using \zcref{def: causal function and ACE}, for all $M \in \mathcal{M} (M_0)$ that
    \begin{equation}
        \theta^{M} \equiv \theta^{M_0} \quad \text{a.e.}
    \end{equation}
    The causal function is identifiable according to \zcref{def: identifiability}.
\end{proof}

\begin{proof}[of \zcref{prop: all identifiability results}-\zcref{prop: linear Gaussian}.]
    By \zcref{app:linear_Gaussian_setting} and \zcref{eq: reg on X and U}, we have for all $x \in \mathcal{X} \setminus \{0\}$ that
    \begin{align}
        \frac{\sigma_{X,Y}\sigma_{U,U} - \sigma_{U,Y} \sigma_{U,X}}{\sigma_{X,X} \sigma_{U,U} - (\sigma_{U,X})^2} = \frac{\theta(x)}{x}.
    \end{align}
    Following \citet{Kuroki2014} and considering the case $\alpha_{UW} \neq 0$, we reformulate the ordinary least squares solution, using that $\sigma_{Y,U} = \frac{\sigma_{Y,W}}{\alpha_{UW}}$, $\sigma_{U,X} = \frac{\sigma_{W,X}}{\alpha_{UW}}$ and $\sigma_{W,W} = \alpha_{UW}^2 \sigma_{U,U} + \sigma_{E}^2$, such that for all $x \in \mathcal{X}$
    \begin{align}
    \label{eq: final equation proof lin gaussian}
        \theta (x) &= \frac{\sigma_{X,Y}\sigma_{U,U} - \sigma_{U,Y} \sigma_{U,X}}{\sigma_{X,X} \sigma_{U,U} - (\sigma_{U,X})^2} x\\
        &= \frac{\sigma_{X,Y} - \frac{\sigma_{U,Y} \sigma_{U,X}}{\sigma_{U,U}} }{\sigma_{X,X}- \frac{(\sigma_{U,X})^2}{\sigma_{U,U}}} x\\
        &= \frac{\sigma_{X,Y} - \frac{\sigma_{W,Y} \sigma_{W,X}}{\alpha_{UW}^2 \sigma_{U,U}} }{\sigma_{X,X}- \frac{\sigma_{W,X}^2}{\alpha_{UW}^2\sigma_{U,U}}} x\\
        &= \frac{\sigma_{X,Y} - \frac{\sigma_{W,Y} \sigma_{W,X}}{\sigma_{W,W} - \sigma_{E}^2} }{\sigma_{X,X}- \frac{\sigma_{W,X}^2}{\sigma_{W,W} - \sigma_{E}^2}} x.
    \end{align}
    For the PC-SCM $M^{\text{lin}}_0$, we reintroduce notation such that for all $x \in \mathcal{X}$ 
    \begin{equation}
    \label{eq: causal function linear model}
        \theta^{M^{\text{lin}}_0} (x)= \frac{\sigma_{X,Y}^{M^{\text{lin}}_0} - \frac{\sigma_{W,Y}^{M^{\text{lin}}_0} \sigma_{W,X}^{M^{\text{lin}}_0}} {\sigma_{W,W}^{M^{\text{lin}}_0} - (\sigma_{E}^2)^{M^{\text{lin}}_0}} }{\sigma_{X,X}^{M^{\text{lin}}_0}- \frac{(\sigma_{W,X}^2)^{M^{\text{lin}}_0}}{\sigma_{W,W}^{M^{\text{lin}}_0} - (\sigma_{E}^2)^{M^{\text{lin}}_0}}} x.
    \end{equation}
    For all observationally equivalent models $M^{\text{lin}} \in \mathcal{M} (M_0^{lin})$, the quantities on the right-hand side of \zcref{eq: causal function linear model} are equivalent. It follows that for all $M^{\text{lin}} \in \mathcal{M} (M_0^{lin})$ 
    \begin{equation}
       \theta^{M^{\text{lin}}} \equiv \theta^{M^{\text{lin}}_0}
    \end{equation}
    and $\theta^{M^{\text{lin}}_0}$ is identifiable according to \zcref{def: identifiability}.
\end{proof}

\section{A Complete Error Mechanism for Continuous Confounders and Proxies}
\label{app:proof_corollary_complete_continuous}

We prove \zcref{thm: continuous confounder} and introduce the notion of a Fourier transform.

\begin{definition}[Fourier transform]
\label{def: fourier transform}
    The Fourier transform $\hat{f}$ of a function $f \in L_1 (\mathbb{R}^d)$ is defined for all $t\in\mathbb{R}^d$ by
    \begin{equation}
        \hat{f}(t) := \left( 2 \pi \right)^{- \frac{d}{2}} \int  f(x) \exp( - i t^{T} x) \, dx
    \end{equation}
    with imaginary unit $i$.
\end{definition}

Now we prove \zcref{thm: continuous confounder} and restate SPICE \zcref{assumption: SPICE}.
\spiceassumption*
\thmcontinuous*
\begin{proof}[of \zcref{thm: continuous confounder}]
\label{proof: corollary identifiability fourier}
For a Borel measurable and injective function $b: \mathbb R^k \rightarrow \mathbb R^d$ and for almost all $(u,w) \in \mathbb{R}^k \times \mathbb{R}^d$, we have that 
\begin{align}
\label{eq:reformulate_error_mech}
    p^{M_0}_{W\mid U}(w\mid u) = p^{M_0}_{E\mid U}(w - b(u) \mid u) = p^{M_0}_{E}(w - b(u) ) 
\end{align}
by the change of variables formula (see for example \citet[Equation 2.89]{Murphy2012}) and the independence of $E$ and $U$.

Let $\delta'\in L_1(\mathcal U)$ satisfy the premise of
$L_1$-completeness in \zcref{def: complete error}, and extend $\delta'$ by zero to
$\mathbb R^k$ such that
\begin{align}
\label{eq:definition-delta}
        \delta(u)
    :=
    \begin{cases}
        \delta'(u), & u\in\mathcal U,\\
        0,          & u\notin\mathcal U.
    \end{cases}
\end{align}
We have that $\delta\in L_1(\mathbb R^k)$ and $ \|\delta\|_{L_1(\mathbb R^k)} = \|\delta'\|_{L_1(\mathcal U)}$ and we define an operator $T: L_1(\mathbb{R}^k) \rightarrow  L_1(\mathbb{R}^d)$ for $w \in \mathbb R ^d$ via
\begin{equation}
    (T \delta) (w)
    : =
    \int_{\mathbb R^k}
        p_E^{M_0}(w-b(u))\delta(u)\,du.
    \label{eq:T-delta-definition}
\end{equation}
The operator $T$ is well defined because, by Fubini's Theorem (for example \citet[Theorem 8.8 (a)]{Rudin1987}) we have that
\begin{align}
    ||T \delta||_{L_1(\mathbb{R}^d)} &\leq \int_{\mathbb R^d}
      \int_{\mathbb R^k}
        p^{M_0}_E(w-b(u))|\delta(u)|\,du\,dw \\
        &=
      \int_{\mathbb R^k}
        |\delta(u)|
        \left(
            \int_{\mathbb R^d}p^{M_0}_E(w-b(u))\,dw
        \right)du
    \\
    &=
      \int_{\mathbb R^k}|\delta(u)|\,du \\
    &=
      \|\delta\|_{L_1(\mathbb{R}^k)} \\
    &<\infty.
    \label{eq:T-delta-absolute-integrability}
\end{align}
The penultimate equality follows from translation invariance of
Lebesgue measure (for example \citet[Theorem 2.42 b.]{Folland1999}) such that for all $u \in \mathcal{U}$
\begin{align}
\label{eq:density_integrates_1}
    \int_{\mathbb R^d}p_E^{M_0}(w-b(u))\,dw = \int_{\mathbb R^d}p_E^{M_0}(w)\,dw = 1
\end{align}
and by SPICE \zcref{assumption: SPICE}-\zcref{as: positive density}. By the premise of $L_1$-completeness in \zcref{def: complete error} and \zcref{eq:reformulate_error_mech,eq:T-delta-definition}, we have that 
\begin{equation}
    T \delta \equiv 0 \quad \text{a.e.}
    \label{eq:T-delta-zero}
\end{equation}

By \cite{Euler1748} and a Pythagorean trigonometric identity, we have for all $(t,w) \in \mathbb{R}^d \times \mathbb{R}^d$ that 
\begin{align}
    |\exp(-it^T w)|= |\cos(t^T w) - i \sin (t^T w)| = \sqrt{\cos^2(t^T w) + \sin^2(t^T w)} =1.
    \label{eq:eulerformula}
\end{align}
For all $t \in \mathbb R^d$, we define the function 
\begin{align}
    h(u,w) := p_E^{M_0}(w-b(u))\delta(u) \exp(-it^Tw), \quad (u,w) \in \mathbb R^k \times \mathbb R^d
\end{align}
and we now show that $h \in L_1(\mathbb R^k \times \mathbb R^d)$. For all $t \in \mathbb R^d$, we have that 
\begin{align}
    \int_{\mathbb R^k \times \mathbb R^d} \left| h(u,w)
    \right| \, d(u,w) &= \int_{\mathbb R^k}\int_{\mathbb R^d} \left| p_E^{M_0}(w-b(u))\delta(u) \exp(-it^Tw)  \right|  \, dw \, du \\
    &= \int_{\mathbb R^k} \left| \delta(u) \right |\int_{\mathbb R^d} p_E^{M_0}(w-b(u))   \, dw \, du \\
    &= \int_{\mathbb R^k} \left| \delta(u) \right| \, du \\
    & < \infty 
\end{align}
where the first equality holds by \citet[Theorem 8.8 (a)]{Rudin1987}, the second equality by \zcref{eq:eulerformula}, the third equality by \zcref{eq:density_integrates_1} and the inequality holds since $\delta \in L_1(\mathbb R^k)$. Since $T\delta \equiv 0$ almost everywhere, we have for all $t \in \mathbb R^d$ that the Fourier transform 
\begin{align}
    \widehat{T \delta} (t) =0.
\end{align}
We have for all $t \in \mathbb R^d$ that
\begin{align}
\widehat{T \delta} (t) 
        &= (2\pi)^{-\frac{d}{2}} \int_{\mathbb R^d} \int_{\mathbb R^k}
        p_E^{M_0}(w-b(u))\delta(u) \exp(-it^Tw) \,du \,dw \\
        &= (2\pi)^{-\frac{d}{2}} \int_{\mathbb R^k} \int_{\mathbb R^d}
        p_E^{M_0}(w-b(u))\delta(u) \exp(-it^Tw) \,dw \,du \\
        &= (2\pi)^{-\frac{d}{2}} \int_{\mathbb R^k} \delta(u) \left(\int_{\mathbb R^d}
        p_E^{M_0}(w-b(u)) \exp(-it^Tw) \,dw \right) \,du
\end{align}
where the second equality follows from \citet[Theorem 8.8 (c)]{Rudin1987} since $h \in L_1(\mathbb R^k \times \mathbb R^d)$. We have for all $t \in \mathbb R^d$ and for almost all $u \in \mathbb R^k$ that 
\begin{align}
    \int_{\mathbb R^d}
        p_E^{M_0}(w-b(u)) \exp(-it^Tw) \,dw &= \int_{\mathbb R^d}
        p_E^{M_0}(w) \exp(-it^T(w + b(u))) \,dw \\
        &= \exp(-it^Tb(u)) \int_{\mathbb R^d}
        p_E^{M_0}(w) \exp(-it^Tw) \,dw \\
        &= \exp(-it^Tb(u)) (2\pi)^{\frac{d}{2}} \hat{p}_E^{M_0}(t)  \\
\end{align}
where the first equality follows from the translation invariance of the Lebesgue measure (for example \citet[Theorem 2.42 b.]{Folland1999}) and the last equality by the definition of the Fourier transform. Thus, we have for all $t \in \mathbb R^d$ that 
\begin{align}
    &\widehat{T \delta} (t)=0 \\
    \iff & (2\pi)^{-\frac{d}{2}} \int_{\mathbb R^k} \delta(u) \exp(-it^Tb(u)) (2\pi)^{\frac{d}{2}} \hat{p}_E^{M_0}(t) \,du =0\\
    \iff & \hat{p}_E^{M_0}(t) \int_{\mathbb R^k} \delta(u) \exp(-it^Tb(u))  \,du =0\\
\end{align}
By SPICE \zcref{assumption: SPICE}-\zcref{as: fourier transform} we have for all $t \in \mathbb{R}^d$ that $\hat{p}_E^{M_0}(t) \neq 0$ and hence it follows for all $t \in \mathbb{R}^d$ that 
\begin{align}
    \int_{\mathbb R^k} \delta(u) \exp(-it^Tb(u))  \,du =0.
\end{align}
We denote the Borel sets of $\mathbb R^d$ by $\mathcal B (\mathbb R^d)$. We define a finite signed measure $\nu$ for all $C \in \mathcal B (\mathbb R^d)$ by
\begin{align}
    \nu(C) := \int_{\mathbb R^k} 1_C(b(u)) \delta (u) \, du
 \end{align}
such that for all $t \in \mathbb{R}^d$
\begin{align}
    \int_{\mathbb R^k} \delta(u) \exp(-it^Tb(u))  \,du =\int_{\mathbb R^d}  \exp(-it^Tz) \, \nu (dz) =0.
\end{align}
By the injectivity of the Fourier transform (see for example \citet[Theorem 5.1 (i)]{Bhattacharya2010}), we have that 
\begin{align}
    \nu  \equiv 0.
\end{align}
We have for all $B \in \mathcal B (\mathbb R^k)$ that $b(B) \in \mathcal B (\mathbb R^d)$ by \citet[Corollary 15.2]{Kechris1995}. Furthermore, we have by injectivity of $b$ that for all $B \in \mathcal B (\mathbb R^k)$ and for all $u \in \mathbb{R}^k$ that
\begin{align}
    1_{b(B)} (b(u)) = 1_{B} (u).
\end{align}
Thus, we have for all $B \in\mathcal B (\mathbb R^k)$ that
\begin{align}
    0 = \nu(b(B)) = \int_{\mathbb R^k} 1_{b(B)}(b(u)) \delta (u) \, du = \int_{\mathbb R^k} 1_{B}(u) \delta (u) \, du =\int_{B}  \delta (u) \, du.
\end{align}
By \citet[Theorem 1.39(b)]{Rudin1987}, we have for almost all $u \in \mathbb{R}^k$ that
\begin{align}
    \delta(u)=0.
\end{align}
By the zero extension in \zcref{eq:definition-delta}, it follows for almost all $u \in  \mathcal{U}$ that
\begin{align}
    \delta'(u)=0.
    \label{eq:delta-zero-final-final}
\end{align}
The function $\delta' \in L_1(\mathcal{U})$ was arbitrary and hence we have shown that for all $\delta'\in L_1(\mathcal{U})$ that
\begin{align}
    \int_{\mathcal{U}} p^{M_0}_{W \mid U} (\cdot \mid u) \delta'(u) \, du \equiv 0 \quad \text{a.e.}
\end{align}
implies
\begin{align}
    \delta' \equiv 0 \quad \text{a.e.}
\end{align}
The error mechanism is $L_1$-complete.
\end{proof}

\section{Sufficient Conditions for a Non-vanishing Fourier Transform}
\label{app:characteristic_divisible}

We prove the result of \zcref{prop: sufficient conditions} and introduce the notions of the characteristic function, infinite divisibility and convolutionally infinitely divisible kernels. The Fourier transform of a probability density function with sign reversal is called the characteristic function of a random variable. We define it with respect to the Lebesgue measure as follows.
\begin{definition}[characteristic function]
\label{def: characteristic function}
    The characteristic function of a $d$-dimensional random variable $X$ is defined for all $t \in \mathbb{R}^d$ by
    \begin{equation}
        \varphi(t) = \int  p_X (x)\exp( i t^{T} x) \, dx
    \end{equation}
    with imaginary unit $i$.
\end{definition}

Now we turn to infinitely divisible characteristic functions. A probability distribution $P$ is infinitely divisible if it can be expressed as the probability distribution of the sum of an arbitrary number of independent and identically distributed random variables. The corresponding characteristic function is called infinitely divisible. We denote the convolution of two distributions by $*$ as defined in \citet[Definition 2.4]{Sato1999}. We define infinitely divisible distributions and characteristic functions based on \citet[Definition 7.1]{Sato1999} as follows.
\begin{definition}[infinitely divisible characteristic function]
\label{def: infinitely divisible characteristic function}
    A probability distribution $P$ on $\mathbb{R}^d$ is infinitely divisible if for all $n \in \mathbb{N}$, there exists a probability distribution $P_n$ on $\mathbb{R}^d$ such that
    \begin{equation}
        P = \underbrace{P_n * P_n * \dots * P_n}_{n \text{ times}}.
    \end{equation}
    The characteristic function of $P$ is called infinitely divisible.
\end{definition}
We define convolutionally infinitely divisible kernels based on \citet{Nishiyama2016}. 
\begin{definition}[convolutionally infinitely divisible kernel]
\label{def: cid kernel}
    A kernel function $k: \mathbb{R}^d \times \mathbb{R}^d \rightarrow \mathbb{R}$ is a convolutionally infinitely divisible kernel if for all $(u,w) \in \mathbb{R}^d \times \mathbb{R}^d$
\begin{align}
    k(u,w) = \kappa(w-u)
\end{align}
where $\kappa$ is a continuous and bounded density of a symmetric and infinitely divisible distribution.    
\end{definition}
Next, we prove \zcref{prop: sufficient conditions}.
\prodivisible*
\begin{proof}[of \zcref{prop: sufficient conditions}]
\label{proof: proposition chara and infinite}
    \begin{enumerate}
        \item \label{prop: non-zero characteristic implies non-zero fourier} Starting from Assumption \zcref{as: non-zero characteristic function} of \zcref{prop: sufficient conditions}, we have for all $t \in \mathbb{R}^d$ that
    \begin{align}
        & \int p^{M_0}_{E} (e) \exp( i t^{T} e) \, de \neq 0 \\
         \iff &(2\pi)^{-\frac{d}{2}}\int p^{M_0}_{E} (e) \exp( i t^{T} e) \, de \neq 0 \label{eq: proof character 1} \\
         \iff &(2\pi)^{-\frac{d}{2}}\int p^{M_0}_{E} (e) \exp( -i t^{T} e) \, de \neq 0 \label{eq: proof character 2}
    \end{align}
    because $(2\pi)^{-\frac{d}{2}}$ is non-zero and we can change the sign of $t$ in \zcref{eq: proof character 1}. The statement in \zcref{eq: proof character 2} is SPICE \zcref{assumption: SPICE}-\zcref{as: fourier transform}.
    \item \label{prop: inf div implies non-zero char} By \citet[Lemma 7.5]{Sato1999}, we have that Assumption \zcref{as: infinite divisibility} of \zcref{prop: sufficient conditions} implies for all $t \in \mathbb{R}^d$ that 
    \begin{align}
        \int \exp(it^Te) \, P^{M_0}_E (de) \neq 0.
    \end{align}
    In combination with SPICE \zcref{assumption: SPICE}-\zcref{as: positive density}, we have for all $t \in \mathbb{R}^d$ that 
    \begin{align}
        \int \exp(it^Te) p^{M_0}_E(e) \, de \neq 0.
    \end{align}
    By the proof in \zcref{prop: non-zero characteristic implies non-zero fourier} above, this is equivalent to SPICE \zcref{assumption: SPICE}-\zcref{as: fourier transform}.
    \item We define for all $(u,w) \in \mathbb{R}^d \times \mathbb{R}^d$ a kernel function $k^{M_0}: \mathbb{R}^d \times \mathbb{R}^d \rightarrow \mathbb{R}$ by
    \begin{align}
        k^{M_0}(u,w) := p^{M_0}_{E}(w-u).
    \end{align}
    We have that $p^{M_0}_E$ is a continuous and bounded density of a symmetric and infinitely divisible distribution $P^{M_0}_E$ by \zcref{def: cid kernel}. This implicitly assumes that $d=k$ and that Assumption \zcref{as: infinite divisibility} of \zcref{prop: sufficient conditions} holds. Since $P^{M_0}_E$ is an infinitely divisible distribution, we have by the proof in \zcref{prop: inf div implies non-zero char} above that this implies that SPICE \zcref{assumption: SPICE}-\zcref{as: fourier transform} holds.
    \end{enumerate}
\end{proof}

\section{When the Error Mechanism is Not Complete}
\label{app_non injective function}

We define non-injective functions in \zcref{def: non-injective function} and prove that the corresponding error mechanism is not complete.
\begin{definition}[non-injective function of positive and finite measure]
    \label{def: non-injective function}
    A measurable function $g: \mathcal{U} \rightarrow \mathbb{R}^d$ is non-injective of positive and finite measure if there exist two disjoint sets $\mathcal{U}_1 \subseteq \mathcal{U}$ and $\mathcal{U}_2 \subseteq \mathcal{U}$ with $\infty >\mu(\mathcal{U}_1) , \mu(\mathcal{U}_2) > 0$ and $g_*(\mu \restriction \mathcal{U}_1)= g_*(\mu \restriction \mathcal{U}_2)$, where $g_*(\mu \restriction \mathcal{U}_1)$ is the pushforward measure of $\mu$ by $g$ restricted on $\mathcal{U}_1$.
\end{definition}
\prononinjective*
\begin{proof}[of \zcref{prop: non injective error mechanism}]
    Let $\mathcal{U}_1$ and $\mathcal{U}_2$ be the measurable sets as defined in \zcref{def: non-injective function}. We define the function $\delta$ by
    \begin{align}
    \delta(u) = 
    \begin{cases}
    1 & \text{if } u \in \mathcal{U}_1 \\
    -1 & \text{if } u \in \mathcal{U}_2 \\
    0 & \text{otherwise.} 
    \end{cases}
    \end{align}
    We have that $\delta \not\equiv 0$ and $\| \delta \|_{L_1(\mathcal{U})} = \mu(\mathcal{U}_1) + \mu(\mathcal{U}_2) <\infty$ by construction and hence $\delta \in L_1(\mathcal{U})$. Further, we have for almost all $(u,w) \in \mathbb{R}^k \times \mathbb{R}^d$ that
    \begin{equation}
    \label{eq:dens_formulations}
        p^{M_0}_{W \mid U}(w \mid u) =p^{M_0}_{E \mid U}(w - g(u) \mid u)  =p^{M_0}_{E}(w - g(u))
    \end{equation}
    by the change of variables formula (see for example Equation 2.89 in \citet{Murphy2012}) and the independence of $E$ and $U$. By the assumption that $g_*(\mu \restriction \mathcal{U}_1)= g_*(\mu \restriction \mathcal{U}_2)$ and \citet[Theorem 14.1]{Schilling2005}, we have for almost all $w \in \mathbb{R}^d$ that
    \begin{align}
        \int_{\mathcal U_1}
            p^{M_0}_E(w-g(u))\,\mu(du)
        &=
        \int_{\mathbb R^d}
            p^{M_0}_E(w-z)\,
            g_{\#}\!\left(
                \mu\!\restriction_{\mathcal U_1}
            \right)(dz)
        \notag\\
        &=
        \int_{\mathbb R^d}
            p^{M_0}_E(w-z)\,
            g_{\#}\!\left(
                \mu\!\restriction_{\mathcal U_2}
            \right)(dz)
        \notag\\
        &=
        \int_{\mathcal U_2}
            p^{M_0}_E(w-g(u))\,\mu(du).
        \label{eq:equal-pushforward-integrals}
    \end{align}
    We now show that these integrals are finite. We have for $i \in \{1,2\}$ that 
    \begin{align}
         \int_{\mathbb R^d} \int_{\mathcal U_i}
                p^{M_0}_E(w-g(u))\,\mu(du)\, \mu(dw)
            &= \int_{\mathbb R^d} \left(\int_{\mathbb R^d}
                p^{M_0}_E(w-z)\,g_*(\mu \restriction \mathcal{U}_i)(dz) \right)\, \mu(dw) \\
            &=\int_{\mathbb R^d} \left(\int_{\mathbb R^d}
                p^{M_0}_E(w-z) \, \mu(dw)\right) \,g_*(\mu \restriction \mathcal{U}_i)(dz)\\
                &=\int_{\mathbb R^d} 1 \,g_*(\mu \restriction \mathcal{U}_i)(dz) \\
                &= g_*(\mu \restriction \mathcal{U}_i)(\mathbb{R}^d) \\
                &= \mu(\mathcal{U}_i) \\
                &< \infty
    \end{align}
    where the second equality follows from Fubini's theorem (for example by \citet[Theorem 8.8 (a)]{Rudin1987}) and the third equality from the translation invariance of Lebesgue measure (for example \citet[Theorem 2.42 b.]{Folland1999}). 
    Hence, for almost all $w \in \mathbb{R}^d$, it holds that
    \begin{align}
        \int_{\mathcal{U}} p^{M_0}_{W \mid U}(w \mid u)\delta (u) \, \mu(du) &= \int_{\mathcal{U}} p^{M_0}_{E}(w - g(u))\delta (u) \, \mu(du) \\
        &= \int_{\mathcal U_1}
            p^{M_0}_E(w-g(u))\,\mu(du) - \int_{\mathcal U_2}
            p^{M_0}_E(w-g(u))\,\mu(du) \\
            &=0
    \end{align}
    by \zcref{eq:dens_formulations,eq:equal-pushforward-integrals}. The premise in \zcref{def: complete error} is satisfied, but since $\delta \not\equiv 0$, the error mechanism is not $L_1$-complete.
\end{proof}

\section{Guarantees for SPICE-Net}
\label{app:confounding equivalence}

\proSPICE*
\begin{proof}[of \zcref{prop: population guarantee SPICE-Net}]
     Since the model for $W \mid (X,Y)$ is correctly specified we have by \citet[Proposition 1]{Shen2023}, which ultimately relies on the energy score being a strictly proper scoring rule \citep[Theorem 2]{Szekely2002}, 
     for almost all $(x,y) \in \mathcal{X} \times \mathcal{Y}$ that
    \begin{align}
    \label{eq:spicenet-w-xy}
        g_\gamma^*(x,y,\mathscr{E},E) \sim  P^{M_0}_{W \mid (X,Y) = (x,y)}.
    \end{align}
    By SPICE \zcref{assumption: SPICE}-\zcref{as: additive error}, there exists a Borel measurable and injective function $b: \mathbb R^k \rightarrow \mathbb R^d $ such that
    \begin{align}
    \label{eq:spice-ass-error}
        W =  b(U) + E.
    \end{align}
    By the mutual independence assumption of the noise terms $(N_U,E,N_X,N_Y)$ of a PC-SCM and SPICE \zcref{assumption: SPICE}-\zcref{as: additive error}, we have that 
    \begin{align}
    \label{eq:indep-e-au}
        E \indep b(U) \mid (X,Y).
    \end{align}
    Hence, by \zcref{eq:spice-ass-error,eq:indep-e-au}, it holds for almost all $(x,y) \in \mathcal{X} \times \mathcal{Y}$ that
    \begin{align}
    \label{eq:chara-fct-1}
        \rho_{W \mid (X,Y) = (x,y)} = \rho_{b(U) \mid (X,Y) = (x,y)} \cdot \rho_{E}
    \end{align}
    where $\rho_W$ is the characteristic function of $W$. Combining \zcref{eq:spicenet-w-xy,eq:chara-fct-1}, we have for almost all $(x,y) \in \mathcal{X} \times \mathcal{Y}$ that
    \begin{align}
        \rho_{g_\gamma^*} = \rho_{b(U) \mid (X,Y) = (x,y)} \cdot \rho_{E}.
    \end{align}
    Since we add noise $E$ independently in the last layer of SPICE-Net, we have for almost all $(x,y) \in \mathcal{X} \times \mathcal{Y}$ that
    \begin{align}
        \rho_{g_\gamma^*} =   \rho_{g_\gamma^* - E} \cdot \rho_{E}
    \end{align}
    and hence for almost all $(x,y) \in \mathcal{X} \times \mathcal{Y}$ that
    \begin{align}
         \rho_{g_\gamma^* - E} \cdot \rho_{E} = \rho_{b(U) \mid (X,Y) = (x,y)} \cdot \rho_{E}.
    \end{align}
    The characteristic function $\rho_{E}$ has no zeros by SPICE \zcref{assumption: SPICE}-\zcref{as: fourier transform} and \zcref{prop: sufficient conditions}-\zcref{as: non-zero characteristic function} and hence for almost all $(x,y) \in \mathcal{X} \times \mathcal{Y}$, it holds that
    \begin{align}
        \rho_{g_\gamma^* - E}  = \rho_{b(U) \mid (X,Y) = (x,y)}.  
    \end{align}
    By the uniqueness of the characteristic function \citep[Theorem 4.3]{Kallenberg2002}, we have for almost all $(x,y) \in \mathcal{X} \times \mathcal{Y}$ that
    \begin{align}
        g_\gamma^*(x,y,\mathscr{E},E) - E \sim  P^{M_0}_{b(U) \mid (X,Y) = (x,y)}.
    \end{align}
\end{proof}
\proadjustingau*
\begin{proof}[of \zcref{prop: equivalence U and AU adjusting}]
    We define the random variable $Z:=b \circ U$. The causal function is given for almost all $x \in \mathcal{X}$ by 
    \begin{align}
        \theta^{M_0}(x) = \mathbb{E}^{M_0; do(X:=x)}\left[Y \right] = \mathbb{E}^{M_0} \left[ \mathbb{E}^{M_0} \left[ Y \mid U, X=x\right]\right].
    \end{align}
   We consider conditional expectations with respect to the $\sigma$-algebra generated by a random variable. For this, we define the $\sigma$-algebra generated by $U$ as
    \begin{align}
    \label{eq: definition sigma alg U}
        \Sigma_U := \{ U^{-1} (B) \mid B \in \mathcal{B}(\mathbb{R}^k)\}
    \end{align}
    where $\mathcal{B}(\mathbb{R}^k)$ are the Borel sets of $\mathbb{R}^k$. The $\sigma$-algebra generated by $Z$ is given by
    \begin{align}
    \label{eq: definition sigma alg Z}
        \Sigma_Z := \{ Z^{-1} (B) \mid B \in \mathcal{B}(\mathbb{R}^d)\}.
    \end{align}
    Next, we show that
    \begin{align}
        \Sigma_U = \Sigma_Z.
    \end{align}
    For a Borel measurable and injective function $b: \mathbb{R}^k \rightarrow \mathbb{R}^d$ and for all $B \in \mathcal{B}(\mathbb{R}^d)$ we have that
    \begin{align}
        Z^{-1} (B) = (b \circ U)^{-1}(B) = (U^{-1} \circ  b^{-1}) (B) .
    \end{align}
    For all $B \in \mathcal{B}(\mathbb{R}^d)$, we have that $b^{-1}(B) \in \mathcal{B}(\mathbb{R}^k)$ since $b$ is Borel measurable and $(U^{-1} \circ  b^{-1}) (B) \in \Sigma_U$ by definition in \zcref{eq: definition sigma alg U}. We have shown that 
    \begin{align}
    \label{eq: sigmaz subset sigmau}
        \Sigma_Z \subseteq \Sigma_U.
    \end{align}
    Because $b$ is an injective function, there exists a (bijective) left-inverse $b': b(\mathbb{R}^k)\rightarrow \mathbb{R}^k$ such that
    \begin{align}
        U = b' \circ Z.
    \end{align}
    The function $b'$ has an inverse which we denote by $b'': \mathbb{R}^k\rightarrow b(\mathbb{R}^k)$. By \citet[Theorem 15.2]{Kechris1995} we have for all $B \in \mathcal{B}(\mathbb{R}^k)$ that 
    \begin{align}
        b(B) \in \mathcal{B}(\mathbb{R}^d)
    \end{align}
    and 
    \begin{align}
    \label{eq: bprime image}
        b''(B) \in \mathcal{B}(b(\mathbb{R}^k)) \subseteq \mathcal{B}(\mathbb{R}^d).
    \end{align}
    Then, we have for all $B \in \mathcal{B}(\mathbb{R}^k)$ that 
    \begin{align}
        U^{-1} (B) = (b' \circ Z)^{-1} (B) = (Z^{-1} \circ b'')(B).
    \end{align}
    For all $B \in \mathcal{B}(\mathbb{R}^k)$, we have that $b''(B) \in \mathcal{B}(\mathbb{R}^d)$ by \zcref{eq: bprime image} and $(Z^{-1} \circ b'' )(B) \in \Sigma_Z$ by definition in \zcref{eq: definition sigma alg Z}. We have shown that 
    \begin{align}
    \label{eq: sigmau subset sigmaz}
        \Sigma_U \subseteq \Sigma_Z.
    \end{align}
    \zcref{eq: sigmaz subset sigmau} and \zcref{eq: sigmau subset sigmaz} lead to
    \begin{align}
        \Sigma_U = \Sigma_Z.
    \end{align}
    Thus, we have for almost all $x \in \mathcal{X}$ that 
    \begin{align}
        \mathbb{E}^{M_0} \left[ Y \mid U, X=x\right]  =  \mathbb{E}^{M_0} \left[ Y \mid Z, X=x\right]  
    \end{align}
    and for almost all $x \in \mathcal{X}$ that 
    \begin{align}
        \theta^{M_0} (x) = \mathbb{E}^{M_0} \left[ \mathbb{E}^{M_0} \left[ Y \mid U, X=x\right]\right] = \mathbb{E}^{M_0} \left[ \mathbb{E}^{M_0} \left[ Y \mid Z, X=x\right]\right].
    \end{align}
\end{proof}

\section{Simulation}\label{app:simulation}

SPICE-Net is implemented in PyTorch \citep{Paszke2017}. In our simulations, we keep the neural network architecture of the first step of SPICE-Net and SPICE-Net-Approx fixed across all data sets A-D of \zcref{tab: data models}. We only modify the distribution of $E$, which we sample from in the last layer of \zcref{fig: spicenet}, to match the data generating process in \zcref{tab: data models} while taking into account the standardisation with values given in \zcref{tab:spicenet_param}. 

\begin{table}[!htbp]
  \centering
  \resizebox{\linewidth}{!}{%
      \begin{tabular}{lllll}
        \toprule
        & A. Gaussian & B. Binary & C. Exponential & D. Multi-dim. \\
        \midrule
         & $\text{loc} = \left[0\right]$ & $\text{loc} = \left[0\right]$ & $\text{rate} = \left[1.414\right]$ & $\text{loc} = \left[0,\, 0,\, 0\right]$ \\
          & $\text{scale} = \left[0.707\right]$ &  $\text{scale} = \left[0.707\right]$ && $\text{covariance}\_\text{matrix} = \begin{bmatrix} 0.5 & -0.095 & 0.149 \\ -0.095 & 0.6 & 0.151 \\ 0.149 & 0.151 & 0.643 \end{bmatrix}$ \\
        \bottomrule
      \end{tabular}}
      \caption{Parameter values used for SPICE-Net across the four data sets A-D of \zcref{tab: data models}. The values correspond to parameters of a normal (A and B), exponential (C) and multivariate normal distribution (D) as implemented in PyTorch \citep{Paszke2017}.}
      \label{tab:spicenet_param}
\end{table}

For SPICE-Net-Approx, we treat the parameters of $E$ as unknown and estimate them from the data. Both methods are trained on $10$ mini-batches over $4000$ epochs with Adam optimization, an adaptive learning rate with an initial value of $10^{-3}$ and ReLU activation functions. The neural networks have five hidden layers with $10, 15, 25, 15$ and $10$ nodes, respectively, of which five are independent standard Gaussian noise nodes. The output dimension of the second-to-last layer, representing the unobserved confounder up to an injective transformation, is chosen to have the same dimension as the proxy variable. We use \citet{He2015} initialisation for the weights of the neural networks and estimate them by minimising the energy loss with a power parameter of one based on two samples as \href{https://github.com/xwshen51/engression/tree/main}{implemented} by \citet{Shen2023}. For the initial values of the parameters of the distribution of $E$ for SPICE-Net-Approx, we use the fact that the variance of $E$ does not exceed the variance of $W$. Since we standardise the data, the variance of $W$ is one and the initial values are given in \zcref{tab:spicenet_init}.

\begin{table}[!htbp]
  \centering
      \begin{tabular}{lllll}
        \toprule
        & A. Gaussian & B. Binary & C. Exponential & D. Multi-dim. \\
        \midrule
         & $\text{loc} = \left[1\right]$ & $\text{loc} = \left[1\right]$ & $\text{rate} = \left[1\right]$ & $\text{loc} = \left[1,\, 2,\, 3\right]$ \\
          & $\text{scale} = \left[1\right]$ &  $\text{scale} = \left[1\right]$ && $\text{covariance}\_\text{matrix} = \begin{bmatrix} 1 & 0.7 & 0.4 \\ 0.7 & 1 & -0.2 \\ 0.4 & -0.2 & 1 \end{bmatrix}$ \\
        \bottomrule
      \end{tabular}
      \caption{Initial values used for SPICE-Net-Approx across the four data sets A-D of \zcref{tab: data models}. The values correspond to parameters of a normal (A and B), exponential (C) and multivariate normal distribution (D) as implemented in PyTorch \citep{Paszke2017}.}
      \label{tab:spicenet_init}
\end{table}

We use the nonparametric, regression adjustment estimator in \zcref{eq: regression adj} to estimate the causal function in the second step of SPICE-Net and SPICE-Net-Approx as well as for Adj.-W and Adj.-U, following the approach by \citet{Zhang2025}. We keep the parameters of the neural network fixed across the four models and across all data sets. We apply a neural network, using the scikit-learn function 'MLPregressor' by \citet{scikit-learn}, with one hidden layer with $100$ neurons, ReLU activations, an adaptive learning rate and an initial learning rate of $0.01$. 

Moreover, we consider only using the treatment to predict the outcome, which we denote by No Adj. We use a neural network to estimate the causal function that has the same structure as the neural network in the second step of SPICE-Net but only uses the treatment and the outcome.

The test errors of the causal function estimators are presented in \zcref{plot:results_2000,plot:results_5000}, with numerical values in \zcref{tab:results_n2000,tab:results_n5000} for a training size of $2000$ and $5000$, respectively. In all plots, we apply the colour scheme by \citet{Petroff2024}.
 
\begin{figure}[!htbp]
  \centering
    \begin{minipage}{\linewidth}
      \centering
      \begin{minipage}[b]{0.48\linewidth}
        \centering
        \includegraphics[width=\linewidth]{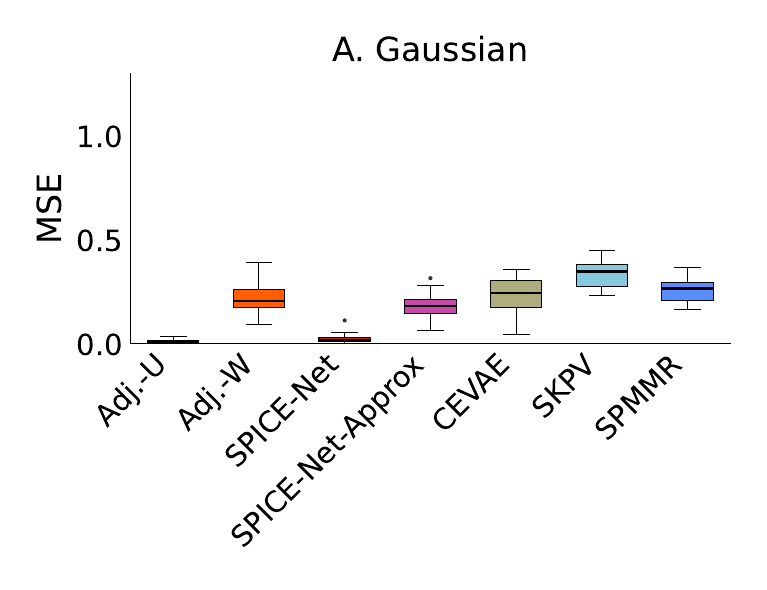}
        
      \end{minipage}
      \hfill
      \begin{minipage}[b]{0.48\linewidth}
        \centering
        \includegraphics[width=\linewidth]{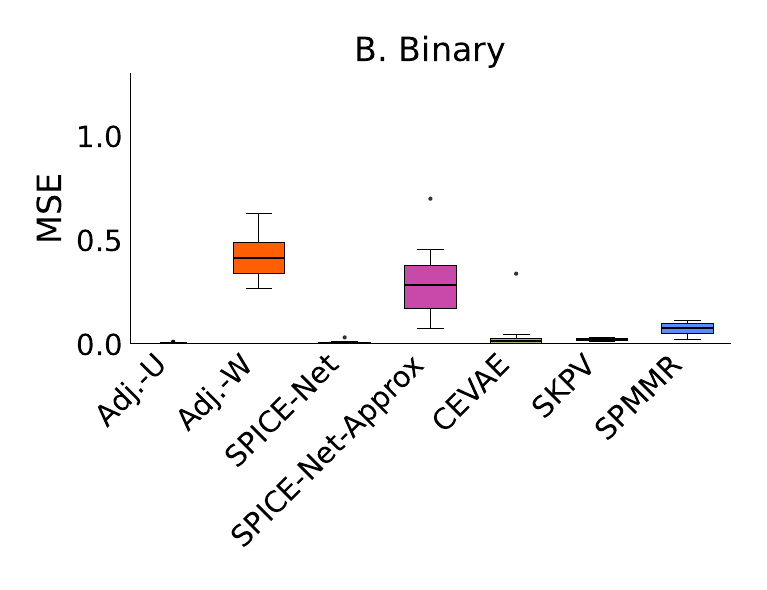}
        
      \end{minipage}

      \vspace{1ex}

      \begin{minipage}[b]{0.48\linewidth}
        \centering
        \includegraphics[width=\linewidth]{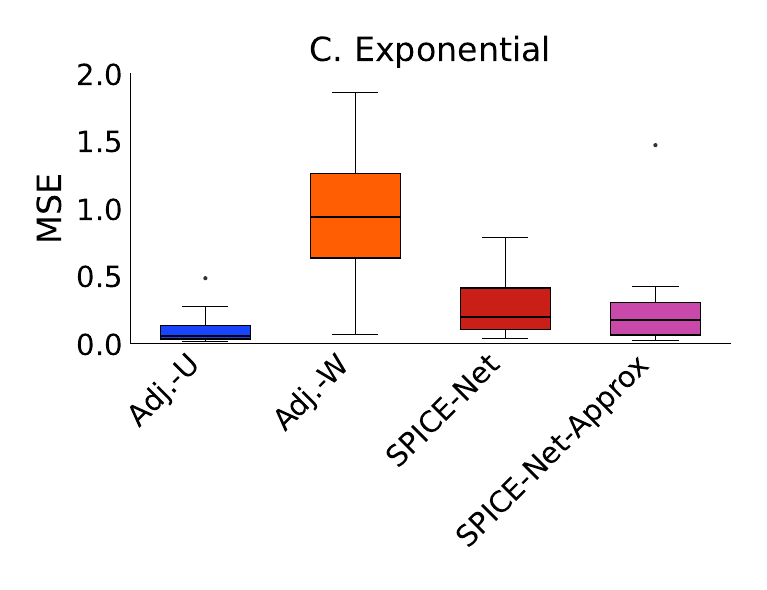}
        
      \end{minipage}
      \hfill
      \begin{minipage}[b]{0.48\linewidth}
        \centering
        \includegraphics[width=\linewidth]{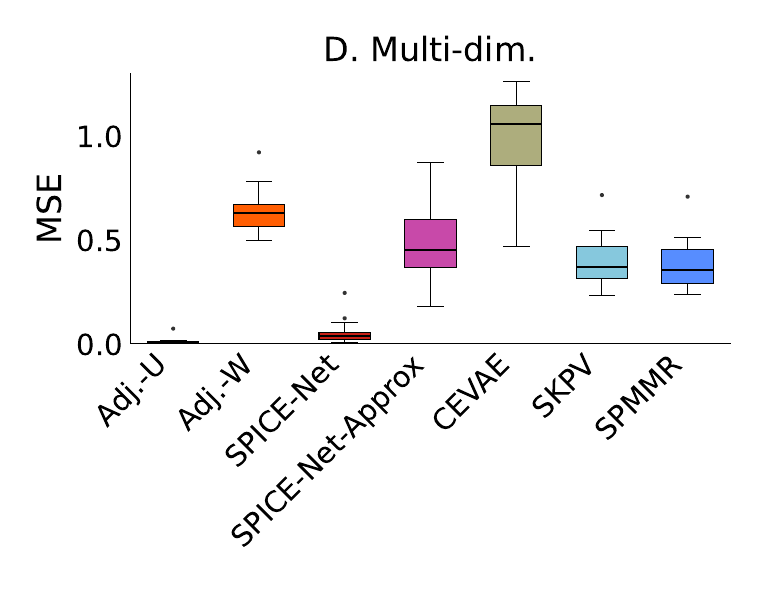}
        
      \end{minipage}

      \caption{Mean squared error (MSE) of causal function estimators described in \zcref{subse: simulations,app:simulation} for $5000$ training samples from data sets A-D of \zcref{tab: data models} on a test set of size $500$. For data set C, the median (standard deviation) of the MSE for CEVAE, SKPV and SPMMR are $3.22~(3.36)$, $12.90~(15.86)$ and $58.53~(18.39)$, respectively, and there is a different scale for the y-axis. Numerical values are given in \zcref{tab:results_n5000}. Our proposed method SPICE-Net achieves an MSE close to the ground-truth method Adj.-U.}
      \label{plot:results_5000}
    \end{minipage}
\end{figure}

\begin{table}[!htbp]
  \centering
  \resizebox{\linewidth}{!}{%
      \begin{tabular}{lllll}
        \toprule
        & A. Gaussian & B. Binary & C. Exponential & D. Multi-dim. \\
        \midrule
        \vspace{0.5em}
        Adj.-U & 0.013 (0.015) & 0.002 (0.006) & 0.132 (0.092) & 0.017 (0.014) \\
        Adj.-W & 0.207 (0.063) & 0.460 (0.089) & 0.862 (0.354) & 0.625 (0.119) \\
        SPICE-Net & \textbf{0.024 (0.016)} & \textbf{0.004 (0.051)} & \textbf{0.276 (0.288)} & \textbf{0.058 (0.057)} \\
        SPICE-Net-Approx & 0.192 (0.062) & 0.256 (0.138) & 0.292 (0.449) & 0.529 (0.112) \\
        CEVAE & 0.268 (0.082) & 0.006 (0.018) & 2.871 (4.702) & 1.053 (0.231) \\
        SKPV & 0.324 (0.060) & 0.018 (0.005) & 16.161 (14.950) & 0.258 (0.126) \\
        SPMMR & 0.265 (0.059) & 0.078 (0.031) & 58.183 (18.205) & 0.353 (0.113) \\
        No Adj. & 0.486 (0.089) & 1.225 (0.194) & 10.857 (5.635) & 1.350 (0.157) \\
        \bottomrule
      \end{tabular}}
      \caption{Median (standard deviation) of the MSE of causal function estimators for training sample size $2000$ from data sets A--D of \zcref{tab: data models} on a test set of size $500$. The lowest MSE in each column, excluding the ground-truth method Adj.-U, is shown in bold. A description of the methods is given in \zcref{subse: simulations,app:simulation} and a graphical representation in \zcref{plot:results_2000}. No Adj.~is a neural network estimator that does not adjust for additional variables.}
      \label{tab:results_n2000}
\end{table}

\begin{table}[!htbp]
  \centering
  \resizebox{\linewidth}{!}{%
      \begin{tabular}{lllll}
        \toprule
        & A. Gaussian & B. Binary & C. Exponential & D. Multi-dim. \\
        \midrule
        \vspace{0.5em}
        Adj.-U & 0.011 (0.009) & 0.001 (0.004) & 0.058 (0.111) & 0.006 (0.015) \\
        Adj.-W & 0.205 (0.074) & 0.415 (0.098) & 0.943 (0.456) & 0.631 (0.099) \\
        SPICE-Net & \textbf{0.017 (0.024)} & \textbf{0.004 (0.007)} & 0.199 (0.247) & \textbf{0.037 (0.054)} \\
        SPICE-Net-Approx & 0.184 (0.061) & 0.285 (0.148) & \textbf{0.177 (0.308)} & 0.453 (0.181) \\
        CEVAE & 0.247 (0.095) & 0.013 (0.072) & 3.216 (3.363) & 1.060 (0.229) \\
        SKPV & 0.349 (0.057) & 0.021 (0.006) & 12.903 (15.859) & 0.372 (0.112) \\
        SPMMR & 0.267 (0.057) & 0.076 (0.027) & 58.529 (18.390) & 0.358 (0.113) \\
        No Adj. & 0.517 (0.133) & 1.313 (0.216) & 12.955 (4.741) & 1.338 (0.211) \\
        \bottomrule
      \end{tabular}}
      \caption{Median (standard deviation) of the MSE of causal function estimators for training sample size $5000$ from data sets A--D of \zcref{tab: data models} on a test set of size $500$. The lowest MSE in each column, excluding the ground-truth method Adj.-U, is shown in bold. A description of the methods is given in \zcref{subse: simulations,app:simulation} and a graphical representation in \zcref{plot:results_5000}. No Adj.~is a neural network estimator that does not adjust for additional variables.}
      \label{tab:results_n5000}
\end{table}

\section{Real-world data} \label{app:chamber}

We run real-world experiments with the Light Tunnel Mk2 from the Causal Chamber\textsuperscript{\tiny®} \citep{Gamella2025} in \zcref{fig: causal chambers setup}. The parameters of the light tunnel for the experiments are given in \zcref{tab: chambers exp details}.

\begin{figure}[!htbp]
  \centering
  \includegraphics[width=\linewidth]{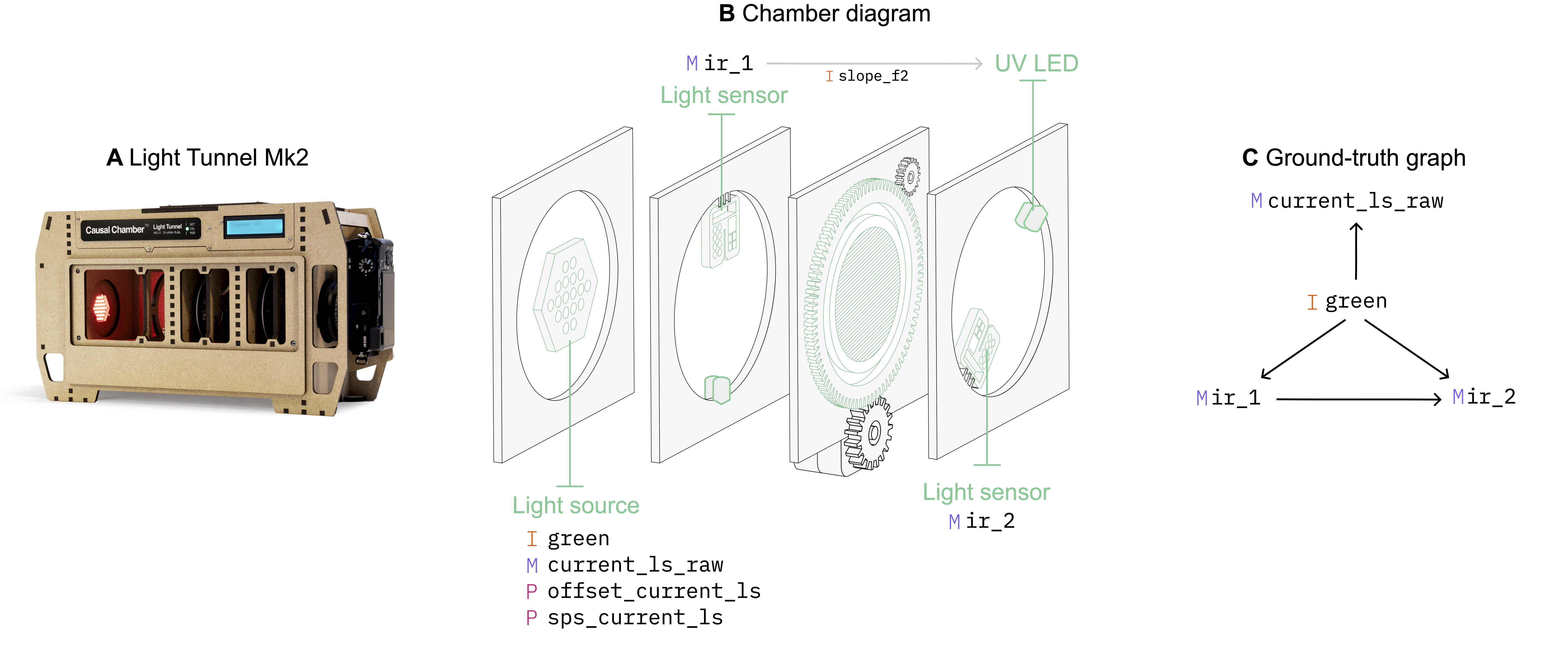}
  \caption{The Light Tunnel Mk2 from Causal Chamber\textsuperscript{\tiny®} (A and B) with its ground-truth graph in C. The variable types are control inputs {\color{IColor} I}, sensor parameters {\color{PColor} P} and sensor measurements {\color{MColor} M}. In the light tunnel, we consider green as the confounder, which is the brightness setting of the green LEDs on the main light source, ir\_1 as the treatment, which is an infrared intensity measurement produced by the first light sensor, ir\_2 as the outcome, which is an infrared intensity measurement produced by the second light sensor and current\_ls\_raw as the proxy, which is a measurement of the electric current drawn by the light source. See \href{https://docs.causalchamber.ai/the-chambers/light-tunnel-mk2}{causalchamber.ai} for a complete documentation.}
  \label{fig: causal chambers setup}
\end{figure}

\begin{table}[!htbp]
      \centering
          \begin{tabular}{llll}
            \toprule
        Variable           & I. Low noise & II. Noisy & III. SPICE-Net noise \\
        \midrule
        green                & $\mathcal{N}^*(60, 15)$ & $\mathcal{N}^*(60, 15)$ & 0 \\
        sps\_current\_ls      & 4 & 7 & $\{4,7\}$ \\
        offset\_current\_ls   & 2950 &  2950 & 2950 \\
        slope\_f2 & 0.064 & 0.064 & 0.064 \\
            \bottomrule
          \end{tabular}
      \caption{The parameter settings of the Light Tunnel Mk2 from Causal Chamber\textsuperscript{\tiny®} of \zcref{fig: causal chambers setup} in our experiments I-III. For the experiments I and II, we generate $5000$ training samples and $500$ test samples $20$ times. The confounder green is generated from a truncated and discretized normal distribution with mean $60$ and standard deviation $15$ denoted by $\mathcal{N}^*(60, 15)$. That is a normal distribution rounded to integers with a minimal value of $0$ and a maximal value of $255$ in order to satisfy the supported values of the variable green. The parameter sps\_current\_ls is the data rate of the Analog-to-Digital Converter (ADC) producing the current\_ls\_raw (proxy) measurement. Lower values lead to more readings to produce the measurement and hence reduce the noise. The parameter offset\_current\_ls determines the reference voltage of the ADC producing the current\_ls\_raw (proxy) measurement and the input slope\_f2 influences the relation between the first infrared intensity measurement ir\_1 (treatment) and the second infrared intensity measurement ir\_2 (outcome). See \href{https://docs.causalchamber.ai/the-chambers/light-tunnel-mk2}{causalchamber.ai} for a complete documentation. In experiment III, we estimate the variance of the noise for the proxy-confounder relation with $5000$ samples, which is used for SPICE-Net. We set sps\_current\_ls to its corresponding value depending on if we estimate the variance for experiment I or II. All other parameters remain constant. The generated data sets of the experiments are available on \href{https://github.com/juangamella/causal-chamber/tree/main/datasets/lt_spice_v1}{github.com/juangamella/causal-chamber}.}
      \label{tab: chambers exp details}
\end{table}

We consider the same methods as in the simulation in \zcref{subse: simulations} with details given in \zcref{app:simulation}. For SPICE-Net, we approximate the noise distribution in the proxy-confounder relation with a zero-mean Gaussian distribution and a variance estimated with the additional experiment III in \zcref{tab: chambers exp details}. Here, we set the confounder to zero and use the variance of the proxy variable current\_ls\_raw as an estimate for the variance of the noise, with values given in \zcref{tab: chambers spicenet params}. We divide the standard deviations in \zcref{tab: chambers spicenet params} by the empirical standard deviation of the proxy in each data set to obtain the parameter values for SPICE-Net. For SPICE-Net-Approx, we initialise the parameters with values of one for both the mean and the standard deviation. Additionally, to prevent SPICE-Net-Approx from driving the learnt variance to zero, we reduce the initial learning rate of SPICE-Net-Approx to $10^{-4}$. All other parameters of all methods are the same as for the simulation in \zcref{subse: simulations} with details given in \zcref{app:simulation}.

\begin{table}[!htbp]
      \centering
          \begin{tabular}{ll}
            \toprule
     I. Low noise & II. Noisy  \\
        \midrule
       13.592 & 37.921  \\
            \bottomrule
          \end{tabular}
      \caption{Estimates for the standard deviation of the noise distribution used for SPICE-Net for the experiments I and II by experiment III of \zcref{tab: chambers exp details}.}
      \label{tab: chambers spicenet params}
\end{table}

The test errors for the causal function estimators relative to Adj.-U are presented in \zcref{tab:chambers_3panel_mse}.

\begin{table}[!htbp]
    \centering
    \begin{tabular}{llll}
    \toprule
        & I. Low noise & II. Noisy \\
        \midrule
        Adj.-W & \textbf{171 (82)} & 3259 (1183) \\
        SPICE-Net & 396 (509) & \textbf{640 (743)} \\
        SPICE-Net-Approx & 833 (1563) & 3180 (1715) \\
        CEVAE & 9661 (2097) & 9817 (2199) \\
        SKPV & 23497 (3042) & 19851 (2696) \\
        SPMMR & 18824 (2918) & 16515 (2489) \\
        No Adj. & 14324 (1644) & 13735 (1791) \\
        \bottomrule
    \end{tabular}
    \caption{Median (standard deviation) of the MSE of each method relative to Adj.-U across the two experiments of \zcref{tab: chambers exp details} on data from the Light Tunnel Mk2 from the Causal Chamber\textsuperscript{\tiny®} \citep{Gamella2025}. The lowest test error for each experiment is shown in bold. The methods are described in \zcref{subse: simulations,subse: chambers,app:simulation,app:chamber} and a graphical representation of the results is given in \zcref{plot:chamber_results}. No Adj.~is a neural network estimator that does not adjust for additional variables.}
      \label{tab:chambers_3panel_mse}
\end{table}

\section{Extensions of SPICE}
\label{app:generalisations}

\paragraph{Additional Observed Confounding} Consider a treatment $X \in \mathcal{X} \subseteq \mathbb{R}^p$, an outcome $Y \in \mathcal{Y} \subseteq \mathbb{R}$, an unobserved confounder $U \in \mathcal{U} \subseteq \mathbb{R}^k$, a proxy variable $W \in \mathcal{W} \subseteq \mathbb{R}^d$ and an observed covariate $O \in \mathcal{O} \subseteq \mathbb{R}^l$, where $p,k,d,l \in \mathbb{N}$. The structural causal model $M$ is given by
\begin{align}
    M:\left\{
            \begin{array}{ll}
            O & \leftarrow N_O \\
            U & \leftarrow N_U \\
            W & \leftarrow f_W\left(U, E\right) \\
            X & \leftarrow f_X \left(O, U, N_X\right) \\
            Y & \leftarrow f_Y \left(O, U, X, N_Y\right)
            \end{array}
            \right.
\label{eq: scm observed O}
\end{align}
with measurable functions $f_W, f_X, f_Y$ and mutually independent random noise terms\\$(N_O, N_U, E, N_X, N_Y)$. We assume that the joint distribution is absolutely continuous with respect to a $\sigma$-finite measure $\mu$ such that it has a density, that is, we have that 
\begin{align}
    p^M(o, u,w,x,y) > 0
    \iff
    (o,u,w,x,y)\in
\mathcal{O} \times \mathcal{U} \times \mathcal{W} \times \mathcal{X} \times \mathcal{Y}.
\end{align}
We assume that we observe $n \in \mathbb{N}$ independent and identically distributed copies\\$(O_1, W_1, X_1, Y_1),\dots,(O_n, W_n,X_n,Y_n)$ of the random variable $(O,W,X,Y)$ generated by a structural causal model $M_0$ from \zcref{eq: scm observed O}. We assume that the error mechanism $\{ p^{M_0}_{W \mid U} (\cdot \mid u)\mid u \in \mathcal{U} \}$ is known. We assume that the causal function, which is given for almost all $x \in \mathcal{X}$ as
\begin{align}
&\int y \int  p^{M_0}_{Y \mid O,U,X} \left(y \mid o, u, x \right) p^{M_0}_{O,U}(o, u) \, \mu(do, du, dy)\\
&= \int y \int   \frac{p^{M_0}_{U \mid O,X,Y} (u \mid o,x,y) p^{M_0}_{O,X,Y}(o,x,y)}{  \int p^{M_0}_{U \mid O,X,Y} (u \mid o,x,y) p^{M_0}_{O,X,Y}(o,x,y)\, \mu(dy) } 
\\ &\left(\int \int p^{M_0}_{U \mid O,X,Y} (u \mid o,x,y) p^{M_0}_{O,X,Y} (o,x,y)\, \mu(dx, dy)\right) \, \mu(do, du, dy),\label{eq: causal function extended}
\end{align}
is uniquely defined and integrable. 

We show that if the error mechanism is $L_1$-complete according to \zcref{def: complete error}, where the SCM is given by \zcref{eq: scm observed O}, the causal function is identifiable.
For almost all $(o,w,x,y) \in \mathcal{O} \times \mathcal{W} \times \mathcal{X} \times \mathcal{Y}$, we have 
\begin{align}
        &p^{M_0}_{W \mid O,X,Y} (w \mid o,x,y) \\& =\int  p^{M_0}_{U,W \mid O, X,Y} (u,w \mid o,x,y) \, \mu(du) \\
        & =\int p^{M_0}_{W \mid O,U,X,Y} (w \mid o,u,x,y) p^{M_0}_{U \mid O,X,Y} (u \mid o, x, y) \, \mu(du) \\
        & =\int  p^{M_0}_{W \mid U} (w \mid u) p^{M_0}_{U \mid O, X,Y} (u \mid o, x, y) \, \mu(du)    \label{eq: proof thm identifiability obs 1}
\end{align}    
where the third equation follows from the independence $W \indep (O, X, Y) \mid U$. The observational features are denoted by 
\begin{align}
    \mathcal{F}^M := \{p^M_{O,W,X,Y}\} \cup \{p^M_{W \mid U}(\cdot \mid u) \mid u \in \mathcal{U}\}.
\end{align}
For all $M \in \mathcal{M} (M_0)$ we have that for almost all $(o, u, x, y) \in \mathcal{O} \times \mathcal{U} \times \mathcal{X} \times \mathcal{Y}$
    \begin{equation}
     \label{eq: proof thm identifiability obs 2}
    p^{M}_{W \mid U} (\cdot \mid u) = p^{M_0}_{W \mid U }(\cdot \mid u)  \text{ and } p^{M}_{ W \mid O, X,Y}(\cdot \mid o,x,y)  = p^{M_0}_{W \mid O, X,Y} (\cdot \mid o,x,y).
    \end{equation} 
This implies for all $M \in \mathcal{M} (M_0)$ and for almost all $(o,w,x,y) \in \mathcal{O} \times \mathcal{W} \times \mathcal{X} \times \mathcal{Y}$ that
\begin{align}
    \int p^{M_0}_{W \mid U} (w \mid u) p^{M}_{U \mid O, X,Y} (u \mid o,x, y) \, \mu(du) = \int p^{M_0}_{W \mid U} (w \mid u) p^{M_0}_{U \mid O,X,Y} (u \mid o, x, y) \, \mu(du)
\end{align}
and hence
\begin{align}
\label{eq: proof thm identifiability obs 3}
        \int p^{M_0}_{W \mid U} (w \mid u) \left( p^{M}_{U \mid O,X,Y}(u  \mid o, x, y) - p^{M_0}_{U \mid O,X,Y}(u  \mid o, x, y)\right) \, \mu(du) = 0
\end{align}
where we apply \zcref{eq: proof thm identifiability obs 1} and \zcref{eq: proof thm identifiability obs 2}. Considering the difference in \zcref{eq: proof thm identifiability obs 3}, we have for all $M \in \mathcal{M} (M_0)$ and for almost all $(o, x, y) \in \mathcal{O} \times \mathcal{X} \times \mathcal{Y}$ that
\begin{align}
    &\int \left |p^{M}_{U \mid O,X,Y} (u \mid o,x, y) - p^{M_0}_{U \mid O,X,Y} (u \mid o,x, y) \right | \, \mu(du)\\
        & \leq \int \left |p^{M}_{U \mid O,X,Y} (u \mid o,x, y)\right | \, \mu(du) + \int \left |p^{M_0}_{U \mid O,X,Y} (u \mid o,x, y) \right | \, \mu(du)\\
        & = \int p^{M}_{U \mid O,X,Y} (u \mid o,x, y) \, \mu(du) + \int p^{M_0}_{U \mid O,X,Y} (u \mid o,x, y)  \, \mu(du) \\
        & = 2.
\end{align}
Hence for all $M \in \mathcal{M} (M_0)$ and for almost all $(o, x, y) \in \mathcal{O} \times \mathcal{X} \times \mathcal{Y}$ it holds that $p^{M}_{U \mid O,X,Y}(\cdot \mid o,x, y) - p^{M_0}_{U \mid O,X,Y}(\cdot \mid o,x, y)\in L_1(\mathcal{U})$. 
We have for all $M \in \mathcal{M} (M_0)$, using \zcref{def: complete error} and \zcref{eq: proof thm identifiability obs 3}, and for almost all $(o,u, x, y) \in \mathcal{O} \times \mathcal{U} \times \mathcal{X} \times \mathcal{Y}$ that
\begin{align}
\label{eq:equality_dens_appendix}
     p^{M}_{U \mid O, X,Y}(u \mid o, x, y) - p^{M_0}_{U \mid O, X,Y}(u \mid o, x, y) = 0.
\end{align}    
The causal function in \zcref{eq: causal function extended} only depends on observational quantities and on the conditional densities of the confounder given the additional covariates, the treatment and the outcome. By the equality of observational quantities and the established equality in \zcref{eq:equality_dens_appendix}, the causal function in \zcref{eq: causal function extended} is identifiable.

\paragraph{Noisy Treatment and Outcome}
Consider a treatment $X \in \mathcal{X} \subseteq \mathbb{R}^p$, an outcome $Y \in \mathcal{Y} \subseteq \mathbb{R}$, an unobserved confounder $U \in \mathcal{U} \subseteq \mathbb{R}^k$, a confounder-proxy variable $W \in \mathcal{W} \subseteq \mathbb{R}^d$, a treatment-proxy variable $X^* \in \mathcal{X^*} \subseteq \mathbb{R}^l$ and an outcome-proxy variable $Y^* \in \mathcal{Y^*} \subseteq \mathbb{R}$, where $p,k,d,l \in \mathbb{N}$. The structural causal model $M$ is given by
\begin{align}
    M:\left\{
            \begin{array}{ll}
            U & \leftarrow N_U \\
            W & \leftarrow f_W\left(U, E\right) \\
            X & \leftarrow f_X \left(U, N_X\right) \\
            X^* & \leftarrow f_{X^*}\left(X, N_{X^*}\right) \\
            Y & \leftarrow f_Y \left(U, X, N_Y\right) \\
            Y^* & \leftarrow f_{Y^*}\left(Y, N_{Y^*}\right) 
            \end{array}
            \right.
\label{eq: scm all noisy}
\end{align}
with measurable functions $f_W, f_X,f_{X^*},f_Y,f_{Y^*}$ and mutually independent random noise terms $(N_U, E, N_X, N_{X^*}, N_Y, N_{Y^*})$. We assume that the joint distribution is absolutely continuous with respect to a $\sigma$-finite measure $\mu$ such that
\begin{align}
    p^M(u,w,x,x^*,y,y^*) > 0
    \iff
    (u,w,x,x^*,y,y^*)\in
\mathcal{U} \times \mathcal{W} \times \mathcal{X} \times \mathcal{X^*} \times \mathcal{Y} \times \mathcal{Y^*}.
\end{align}
We assume that we observe $n \in \mathbb{N}$ independent and identically distributed copies\\$( W_1, X^*_1, Y^*_1),\dots,(W_n,X^*_n,Y^*_n)$ of the random variable $(W,X^*,Y^*)$ generated by a structural causal model $M_0$ from \zcref{eq: scm all noisy}. We assume that the error mechanisms 
\begin{align}
    \{ p^{M_0}_{W \mid U} (\cdot \mid u)\mid u \in \mathcal{U}\} \cup \{ p^{M_0}_{X^* \mid X} (\cdot \mid x)\mid x \in \mathcal{X}\} \cup \{ p^{M_0}_{Y^* \mid Y} (\cdot \mid y)\mid y \in \mathcal{Y}\}
\end{align}
are known. The observational features in this extended setting are denoted by
\begin{align}
    \mathcal{F}^M :=  \{ p^{M}_{W,X^*,Y^*}\} \cup \{ p^{M}_{W \mid U} (\cdot \mid u)\mid u \in \mathcal{U}\} \cup \{ p^{M}_{X^* \mid X} (\cdot \mid x)\mid x \in \mathcal{X}\} \cup \{ p^{M}_{Y^* \mid Y} (\cdot \mid y)\mid y \in \mathcal{Y}\}.
\end{align}

We show that if all three error mechanisms are $L_1$-complete, the causal function in \zcref{def: causal function and ACE} for a SCM of \zcref{eq: scm all noisy} is identifiable. We start by considering the error mechanism for the relation between $W$ and $U$. From the proof of \zcref{thm: effect identifiable} in \zcref{proof: theorem identifiability}, and specifically from \zcref{eq: proof thm identifiability 3.2} where we swap $X$ with $X^*$ and $Y$ with $Y^*$, it follows for all $M \in \mathcal{M}(M_0)$ and for almost all $(w,x^*,y^*) \in \mathcal{W} \times \mathcal{X^*} \times \mathcal{Y^*}$ that
\begin{align}
    \int p^{M_0}_{W \mid U} (w\mid u) \left( p^M_{U \mid X^*, Y^*} (u \mid x^*,y^*) -  p^{M_0}_{U \mid X^*, Y^*} (u \mid x^*,y^*)\right) \, \mu(du)= 0.
\end{align}
Since we have that $\{ p^{M_0}_{W \mid U} (\cdot \mid u)\mid u \in \mathcal{U}\}$ is $L_1(\mathcal{U})$-complete, it follows that for all $M \in \mathcal{M} (M_0)$ and for almost all $(u, x^*,y^*) \in \mathcal{U} \times \mathcal{X^*} \times \mathcal{Y^*}$ that 
\begin{align}
    p^{M}_{U \mid X^*, Y^*}(u \mid x^*, y^*) = p^{M_0}_{U \mid X^*, Y^*}(u \mid x^*, y^*)
\end{align} 
and 
\begin{align}
    p^{M}_{U , X^*, Y^*}(u , x^*, y^*) = p^{M_0}_{U , X^*, Y^*}(u , x^*, y^*).
    \label{eq: equivalance density u xstar ystar}
\end{align} 
For almost all $(u, x^*,y^*) \in \mathcal{U} \times \mathcal{X^*} \times \mathcal{Y^*}$, we have that
\begin{align}
    p^{M_0}_{U , X^*, Y^*}(u , x^*, y^*) &= \int p^{M_0}_{U , X,X^*, Y^*}(u , x,x^*, y^*) \, \mu(dx) \\
    &= \int p^{M_0}_{X^* \mid X}(x^* \mid x) p^{M_0}_{U ,X, Y^*}(u , x, y^*) \, \mu(dx) \\
    &= \int p^{M_0}_{X^* \mid X}(x^* \mid x) p^{M_0}_{X \mid U, Y^*}(x \mid u, y^*) p^{M_0}_{U, Y^*}(u, y^*) \, \mu(dx) 
\end{align}
where the second equality holds because $X^* \indep (U,Y^*) \mid X$. By the definition of $\mathcal{M}(M_0)$ and by \zcref{eq: equivalance density u xstar ystar}, it holds for all $M \in \mathcal{M}(M_0)$ and for almost all $(u, x, x^*,y^*) \in \mathcal{U} \times \mathcal{X} \times \mathcal{X^*} \times \mathcal{Y^*}$ that
\begin{align}
    &p^{M}_{U , X^*, Y^*}(u , x^*, y^*) = p^{M_0}_{U , X^*, Y^*}(u , x^*, y^*) \text{ , } p^{M}_{X^* \mid X}(x^* \mid x) = p^{M_0}_{X^* \mid X}(x^* \mid x) \text{ and } \\ 
    &p^{M}_{U, Y^*}(u, y^*) = p^{M_0}_{U, Y^*}(u, y^*).
\end{align}
This implies for all $M \in \mathcal{M}(M_0)$ and for almost all $(u, x, x^*,y^*) \in \mathcal{U} \times \mathcal{X} \times \mathcal{X^*} \times \mathcal{Y^*}$  that 
\begin{align}
    &\int p^{M_0}_{X^* \mid X}(x^* \mid x) p^{M_0}_{X \mid U, Y^*}(x \mid u, y^*) p^{M_0}_{U, Y^*}(u, y^*) \, \mu(dx)\\ =&\int p^{M_0}_{X^* \mid X}(x^* \mid x) p^{M}_{X \mid U, Y^*}(x \mid u, y^*) p^{M_0}_{U, Y^*}(u, y^*) \, \mu(dx) 
\end{align}
and hence 
\begin{align}
    \int p^{M_0}_{X^* \mid X}(x^* \mid x) \left( p^{M}_{X \mid U, Y^*}(x \mid u, y^*) p^{M_0}_{U, Y^*}(u, y^*) - p^{M_0}_{X \mid U, Y^*}(x \mid u, y^*) p^{M_0}_{U, Y^*}(u, y^*) \right) \, \mu(dx) = 0.
\end{align}
By the $L_1(\mathcal{X})$-completeness of $\{ p^{M_0}_{X^* \mid X} (\cdot \mid x )\mid x \in \mathcal{X}\}$, we have for all $M \in \mathcal{M}(M_0)$ and for almost all $(u,x,y^*) \in \mathcal{U} \times \mathcal{X} \times \mathcal{Y^*}$ that
\begin{align}
    p^{M}_{X \mid U, Y^*}(x \mid u, y^*) p^{M_0}_{U, Y^*}(u, y^*) = p^{M_0}_{X \mid U, Y^*}(x \mid u, y^*) p^{M_0}_{U, Y^*}(u, y^*)
\end{align}
and  
\begin{align}
    p^{M}_{U , X, Y^*}(u, x, y^*)= p^{M_0}_{U , X, Y^*}(u, x, y^*).
    \label{eq: equivalance u x ystar}
\end{align}
We have for almost all $(u,x,y^*) \in \mathcal{U} \times \mathcal{X} \times \mathcal{Y}^*$ that 
\begin{align}
    p^{M_0}_{U , X, Y^*}(u, x, y^*) &= \int p^{M_0}_{U , X, Y, Y^*}(u, x, y, y^*) \, \mu(dy) \\
    &=\int p^{M_0}_{Y^* \mid Y}(y^* \mid y) p^{M_0}_{U , X, Y}(u, x, y) \, \mu(dy).
\end{align}
because $Y^* \indep (U,X) \mid Y$. By the definition of $\mathcal{M}(M_0)$ and by \zcref{eq: equivalance u x ystar}, we have for all $M \in \mathcal{M}(M_0)$ and for almost all $(u,x,y^*) \in \mathcal{U} \times \mathcal{X} \times \mathcal{Y^*}$ that 
\begin{align}
\int p^{M_0}_{Y^* \mid Y}(y^* \mid y) p^{M}_{U , X, Y}(u, x, y) \, \mu(dy) = \int p^{M_0}_{Y^* \mid Y}(y^* \mid y) p^{M_0}_{U , X, Y}(u, x, y) \, \mu(dy)    
\end{align}
and hence
\begin{align}
\int p^{M_0}_{Y^* \mid Y}(y^* \mid y) \left( p^{M}_{U , X, Y}(u, x, y) - p^{M_0}_{U , X, Y}(u, x, y) \right) \, \mu(dy) = 0.
\end{align}
By the $L_1(\mathcal{Y})$-completeness of $\{ p^{M_0}_{Y^* \mid Y} (\cdot \mid y )\mid y \in \mathcal{Y}\}$, we have for all $M \in \mathcal{M}(M_0)$ and for almost all $(u,x,y) \in \mathcal{U} \times \mathcal{X} \times \mathcal{Y}$ that 
\begin{align}
    p^{M}_{U , X, Y}(u, x, y) = p^{M_0}_{U , X, Y}(u, x, y).
\label{eq: identitiy joint dens extension}
\end{align}
The causal function from \zcref{def: causal function and ACE} for almost all $x \in \mathcal{X}$ is given by
\begin{align}
        \theta^{M_0} (x) &=  \int y \int  p^{M_0}_{Y \mid U,X} \left(y \mid u, x \right) p^{M_0}_{U}(u) \, \mu(du, dy) \\
        &=  \int y \int   \frac{p^{M_0}_{U,X,Y}(u,x,y)}{p^{M_0}_{U,X} (u,x)} \left(\int  p^{M_0}_{U,X,Y} (u,x,y) \,\mu(dx, dy) \right)\, \mu(du, dy) \\
        &=  \int y \int   \frac{p^{M_0}_{U,X,Y}(u,x,y)}{\int p^{M_0}_{U,X,Y} (u,x,y) \, \mu (dy)} \left(\int  p^{M_0}_{U,X,Y} (u,x,y) \,\mu(dx, dy) \right)\, \mu(du, dy).
\label{eq: cf extension}
\end{align}
In \zcref{eq: cf extension}, the causal function only depends on the joint density of the confounder, the treatment and the outcome. By \zcref{eq: identitiy joint dens extension}, we have for all $M \in \mathcal{M}(M_0)$ that 
\begin{align}
    \theta^M \equiv \theta^{M_0} \quad \text{a.e.}
\end{align}

\clearpage

\bibliography{references}

\end{document}